%% file: draft.tex
\documentclass[letterpaper]{article} 
\usepackage{aaai2026}  
\usepackage{times}  
\usepackage{helvet}  
\usepackage{courier}  
\usepackage[hyphens]{url}  
\usepackage{graphicx} 

\usepackage{natbib}  
\usepackage{caption} 
\usepackage{algorithm}
\usepackage{algorithmicx}

\usepackage{newfloat}
\usepackage{listings}
\DeclareCaptionStyle{ruled}{labelfont=normalfont,labelsep=colon,strut=off} 
\floatstyle{ruled}
\newfloat{listing}{tb}{lst}{}
\floatname{listing}{Listing}
\usepackage{amsfonts}
\usepackage{amsmath}

\usepackage{algpseudocode}
\usepackage{multirow}
\usepackage{enumitem}
\usepackage{booktabs}
\usepackage[table,xcdraw]{xcolor}
\usepackage{adjustbox}
\usepackage{siunitx}
\usepackage{diagbox}
\usepackage{mathtools}
\usepackage{amssymb,amsthm}
\usepackage{tabularx}
\usepackage{placeins}
\usepackage[perpage]{footmisc}

\title{Harmonic Dataset Distillation for Time Series Forecasting}
\author{
    Seungha Hong\textsuperscript{\rm 1}, Sanghwan Jang\textsuperscript{\rm 1}, Wonbin Kweon\textsuperscript{\rm 2}, Suyeon Kim\textsuperscript{\rm 1}, Gyuseok Lee\textsuperscript{\rm 2}, Hwanjo Yu\textsuperscript{\rm 1}\footnote{Corresponding Author}
}
\affiliations{
\textsuperscript{\rm 1}Pohang University of Science and Technology (POSTECH), Republic of Korea\\
\textsuperscript{\rm 2}University of Illinois Urban-Champaign (UIUC), United States\\
\{shhong97, s.jang, kimsu, hwanjoyu\}@postech.ac.kr, \{wonbin, gyuseok2\}@illinois.edu

}

\usepackage{bibentry}

\usepackage{selectp}

\begin{document}

\maketitle

\begin{abstract}

Time Series forecasting (TSF) in the modern era faces significant computational and storage cost challenges due to the massive scale of real-world data. Dataset Distillation (DD), a paradigm that synthesizes a small, compact dataset to achieve training performance comparable to that of the original dataset, has emerged as a promising solution. However, conventional DD methods are not tailored for time series and suffer from architectural overfitting and limited scalability. To address these issues, we propose \underline{H}armonic Dataset \underline{D}istillation for \underline{T}ime Series Forecasting (\textbf{HDT}). HDT decomposes the time series into its sinusoidal basis through the FFT and aligns the core periodic structure by \emph{Harmonic Matching}. Since this process operates in the frequency domain, all updates during distillation are applied globally without disrupting temporal dependencies of time series. Extensive experiments demonstrate that HDT achieves strong cross-architecture generalization and scalability, validating its practicality for large-scale, real-world applications.

\end{abstract}

\section{Introduction}

Time Series Forecasting (TSF), a task that aims to predict future values based on historical data, is a critical task across numerous domains, including industrial manufacturing, healthcare, traffic, and meteorology \cite{survey}. However, the practical application of TSF faces significant challenges related to data storage and computational cost. On the data side, sources such as industrial sensors and biomedical monitors collect data at high frequencies (often every minute or second), producing terabytes of sequential data daily and making it impractical to store the complete historical record \cite{cisco, noaa, vitaldb}. On the model side, the recent advent of large foundation models such as TimesFM \cite{timesfm} and Moirai \cite{moirai} has intensified these computational burdens, amplifying the need for efficient data handling.

\begin{figure}[t!]
  \centering
  \includegraphics[width=\linewidth]{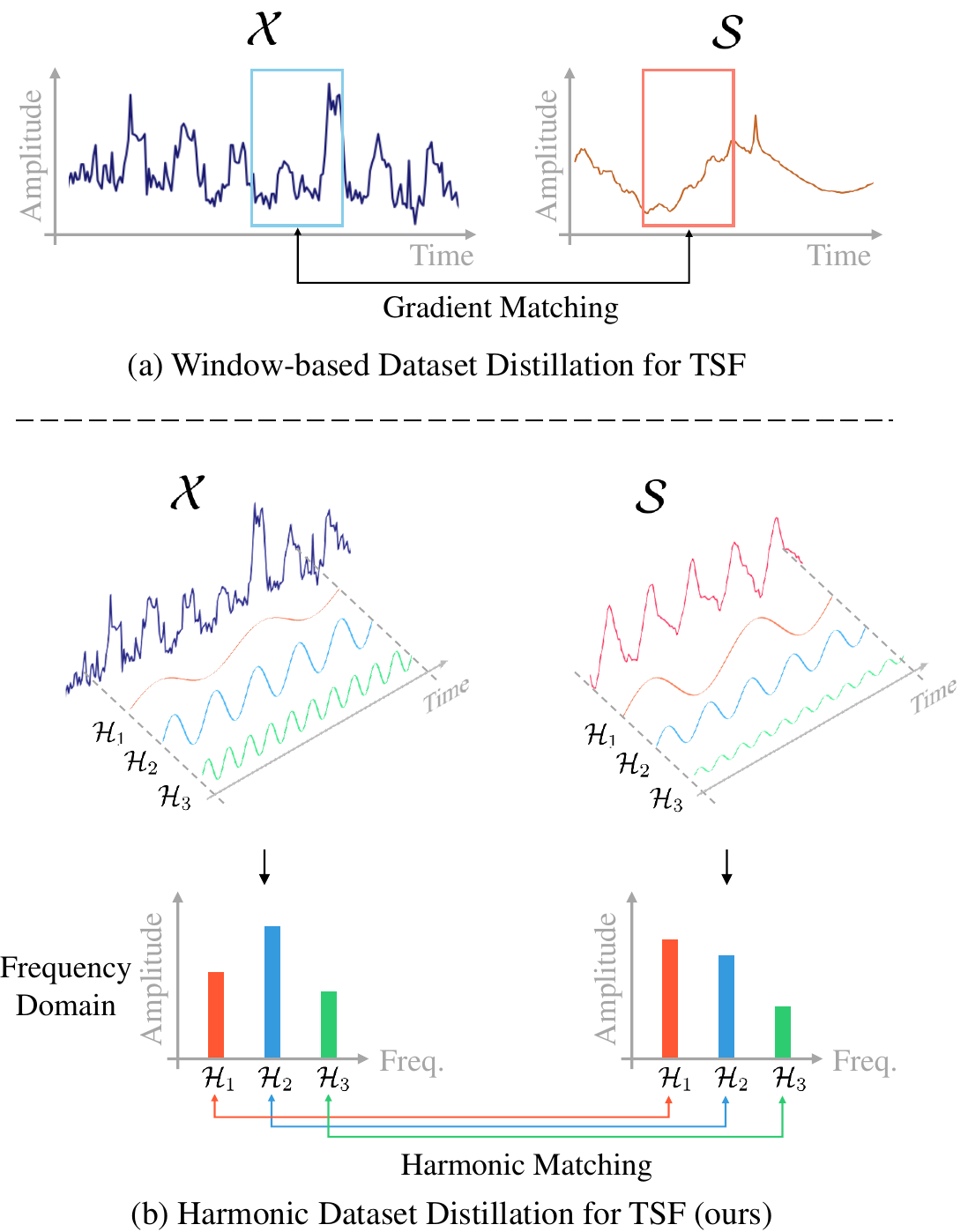}
  \caption{Illustrative examples of distilling an original dataset $\mathcal{X}$ into a synthetic dataset $\mathcal{S}$ with (a) Window-based and (b) Harmonic Dataset Distillation for TSF. In (a), windows randomly sampled from $\mathcal{X}$ are distilled into arbitrary positions within $\mathcal{S}$. In (b), the sequence is decomposed into a sinusoidal basis, and selected harmonics ($\mathcal{H}_i$) between $\mathcal{X}$ and $\mathcal{S}$ are aligned during distillation.
 }
\label{vs}
\end{figure}

A promising approach to address these challenges is \emph{Dataset Distillation} (DD), a paradigm that synthesizes a small, compact dataset whose training performance is comparable to that of the entire original dataset.
Initially formulated as a bi-level optimization problem \cite{wang2018datasetdistillation}, DD has demonstrated its effectiveness in image classification \cite{zhao2021datasetcondensation, nguyen2021kip, cazenavette2022dataset, zhao2023distribution}.
Recently, there have been attempts to extend it to other data modalities such as graphs \cite{jin2022graph, liu2024gcsr}, time series \cite{liu2024dataset, ding2024condtsf}, and natural language \cite{maekawa2023text, maekawa2024dilm}.

However, directly applying conventional DD methods to TSF fails to account for the global structure of time series data. In TSF, models use small, fixed-size windows of data as input and output, even when the overall time series is long. For example, a model might use 96 past steps to predict the next 96 steps while the total time series spans 20,000 to 70,000 data points. If we adopt existing DD methods directly to time series, each time window is treated as an independent data instance. Therefore, optimization is performed within these local windows, trying to match local windows from the synthetic data to those from the original data (Figure~\ref{vs}-(a)). This ``local-to-local" approach, which we refer to as \emph{Window-based Dataset Distillation for TSF}, completely disregards the characteristics of time series, such as long-range dependencies and periodicity. This failure to capture the global context leads to limited scalability with respect to the synthetic data size and degraded cross-architecture generalization.

To tackle these limitations, we propose \emph{\underline{H}armonic Dataset \underline{D}istillation for \underline{T}ime Series Forecasting} (\textbf{HDT}). Instead of performing distillation on localized time domain window pairs, HDT shifts the optimization space to the frequency domain using the Fast Fourier Transform (FFT). Specifically, we apply the FFT to represent both the original and synthetic time series as a sum of sinusoidal basis functions. Among these, we define the dominant components as \emph{harmonics}, which contain the core periodic information of the sequence.
Our \emph{Harmonic Matching} aligns the distributions of these harmonics between the original and synthetic data, thereby preserving the global structure of time series (Figure~\ref{vs}-(b)). Furthermore, since each basis has a global influence over the entire sequence, any updates in this optimization space result in a modification of the synthetic sequence as a whole. This ensures updates are applied without disrupting the temporal dependencies within the synthetic series, even with a conventional distillation loss.

To validate our proposed method, we provide a comprehensive evaluation from both theoretical and empirical standpoints. We first conduct a theoretical analysis to formally justify how \emph{Harmonic Matching} preserves essential temporal structures. Complementing this, our extensive experiments on various TSF datasets confirm that HDT achieves state-of-the-art performance across diverse backbone architectures and synthetic dataset sizes.

In summary, our work makes the following contributions:
\begin{itemize}
    \item We introduce HDT, a novel and effective dataset distillation method for time series forecasting.
    \item A theoretical analysis is provided, proving that the synthetic dataset distilled by HDT preserves the essential global structure of the original data.
    \item Extensive experiments on modern backbones demonstrate the method's state-of-the-art performance and strong cross-architecture generalization.
\end{itemize}

\section{Preliminaries}
We first provide the problem formulation and key notations of Time Series Forecasting (TSF) and Dataset Distillation (DD). Then, we examine the scenario of directly applying conventional DD methods to TSF and discuss their limitations.

\subsection{Time Series Forecasting}
A time series, denoted as $\mathcal{X}\in \mathbb{R}^{N \times C}$, is a sequence of data points measured at successive time intervals.
Here, $N$ is the total length of the time series, and $C$ is the number of variables (or channels) measured at each time step. 
In TSF, given a historical lookback window $\mathcal{X}_{s+1:s+l}$ of length $l$ starting from an arbitrary time step $s$, the objective is to predict a future forecast horizon $\mathcal{X}_{s+l+1:s+l+t}$ of length $t$. 

For model training and evaluation, the entire series $\mathcal{X}$ is divided chronologically into training set $\mathcal{X}_\text{train}$ and test set $\mathcal{X}_\text{test}$. From each of these sets, a sliding window is applied to generate input-output pairs $(x, y)$, where $x = \mathcal{X}_{s+1:s+l}$ and $y = \mathcal{X}_{s+l+1:s+l+t}$ for all possible start times $s$.
With the prepared training samples, the model is trained by minimizing the loss function $\mathcal{L}$:
\begin{equation}
\begin{aligned}
    \mathcal{L}(\theta, \mathcal{X})&:= \frac{1}{|\mathcal{X}|}\sum_{(x, y) \in \mathcal{X}} d(f_\theta(x), y).
\end{aligned}
\end{equation}
Here, the loss $\mathcal{L}$ is the average of a distance metric $d$ (e.g., L2 distance) between the model's predictions $f_\theta(x)$ and the true values $y$ over all pairs in the dataset.

The training process finds the optimal parameters $\theta^*$ by solving the following optimization problem on the training data:
\begin{equation}
\begin{aligned}
    \theta^* & = \underset{\theta}{\text{argmin}}~ \mathcal{L}(\theta, \mathcal{X}_\text{train}).
\end{aligned}
\end{equation}
For convenience, we encapsulate this entire training process as a function $\mathcal{T}$. This function takes the initial parameter $\theta$ and the training dataset $\mathcal{X}_\text{train}$ as inputs, and returns the trained parameters:
\begin{equation}
\begin{aligned}
\mathcal{T}(\theta, \mathcal{X}_\text{train})\approx
 \underset{\theta}{\text{argmin}}~ \mathcal{L}(\theta, \mathcal{X}_\text{train}).
\end{aligned}
\end{equation}

\subsection{Dataset Distillation}
The goal of \emph{Dataset Distillation} \cite{wang2018datasetdistillation} is to create a small synthetic dataset $\mathcal{S}$ that can train a model as effective as the original training dataset $\mathcal{X}_\text{train}$:
\begin{equation}
    \mathcal{L}(\mathcal{T}(\theta, \mathcal{S}), \mathcal{X}_\text{test})\approx\mathcal{L}(\mathcal{T}(\theta, \mathcal{X}_\text{train}), \mathcal{X}_\text{test}).
\end{equation}
This is achieved by optimizing $\mathcal{S}$ to minimize the evaluation loss of ${\mathcal{T}(\theta, \mathcal{S})}$ over $\mathcal{X}_\text{test}$, which can be formulated as the following bi-level optimization problem: 
\begin{equation}
    \begin{aligned}
    \mathcal{S}^* &= \underset{\mathcal{S}}{\text{argmin}} \: \mathcal{L}(\mathcal{T}(\theta, \mathcal{S}), \mathcal{X}_\text{test}).
    \end{aligned}
    \label{biopt}
\end{equation}
Here, $\mathcal{S}$ is treated as learnable parameters and optimized using algorithms such as stochastic gradient descent (SGD).
However, directly optimizing Equation~\ref{biopt} is computationally expensive due to the inclusion of the entire inner-loop training process $\mathcal{T}(\theta, \mathcal{S})$.
To mitigate this, surrogate objectives that approximate the original optimization problem have been proposed \cite{zhao2021datasetcondensation, cazenavette2022dataset, zhao2023distribution}.

In dataset distillation, the effectiveness of a synthetic dataset $\mathcal{S}$ hinges on its ability to effectively train arbitrary models, not just the one used for its creation.
This \textit{cross-architecture generalization} ensures that $\mathcal{S}$ captures the essential knowledge of the original dataset $\mathcal{X}_\text{train}$ rather than merely memorizing a specific model's training process.

\subsection{Dataset Distillation for TSF} 

The objective is to obtain a synthetic sequence $\mathcal{S} \in \mathbb{R}^{M \times C}$ that serves as a compressed substitute for the original time series $\mathcal{X} \in \mathbb{R}^{N \times C}$, where $M \ll N$.\footnote{For simplicity, $\mathcal{X}$ will henceforth be considered to include only the training windows (i.e., $\mathcal{X} := \mathcal{X}_\text{train}$).}
A straightforward way to extend image dataset distillation methods to TSF datasets is to randomly sample windows from $\mathcal{X}$ and $\mathcal{S}$ and treat them as individual instances \cite{zhao2021datasetcondensation, cazenavette2022dataset, cui2023scaling, ding2024condtsf}. Specifically, updating $\mathcal{S}$ involves randomly sampling mini-batches from $\mathcal{S}$ and $\mathcal{X}$, which are represented as follows:
\begin{equation}
\label{window_based}
\mathcal{B_S} \in \mathbb{R}^{|\mathcal{B_S}| \times (l+t)  \times C}, \: \mathcal{B_X} \in \mathbb{R}^{|\mathcal{B_X}| \times (l+t)  \times C}.
\end{equation}
Here, each window has a length of $l+t$ (input window + output window).
The distillation loss (e.g., gradient matching \cite{zhao2021datasetcondensation}) is computed for the sampled mini-batches, and the data points within the sampled windows $\mathcal{B_S}$ are directly updated.
We refer to this approach as \emph{Window-based TSF Dataset Distillation}.

The ``local-to-local" matching of window-based approaches presents critical limitations from two perspectives:

\subsubsection{L1. Limited Scalability} 
Ideally, increasing the size of the synthetic dataset ($M$) should improve performance by capturing more diverse patterns within training dataset. However, with this approach, a larger $M$ merely elongates the existing local patterns rather than capturing the broader, global structure of the time series. Consequently, this leads to diminishing returns, since the added data points provide little valuable information.

\subsubsection{L2. Architectural Overfitting}
The window-based approaches update only the data points within a given window to minimize the specified distillation loss. This local optimization completely ignores the global dependencies that structure the entire time series. With the global context disregarded, the distillation leads the synthetic data to overfit to the specific inductive biases of the fixed backbone model, rather than learning the essential patterns inherent to the original data. Accordingly, this results in a degraded cross-architecture generalization.

\section{Methodology}

To address the limitations of window-based approaches, we propose \emph{\underline{H}armonic Dataset \underline{D}istillation for \underline{T}ime Series Forecasting} (\textbf{HDT}). The core idea is to decompose the time series into its sinusoidal basis through FFT and perform distillation in the frequency domain. Instead of updating data points within local windows, our method directly updates the selected dominant basis functions, which we refer to as \emph{harmonics}. The optimization is guided by two components: (1) \emph{Harmonic Matching} to align the global harmonic distributions, and (2) \emph{Gradient Matching} as a distillation loss. This approach allows HDT to effectively capture global structures and overcome the limitations of window-based methods. An overview of our method is depicted in Figure~\ref{overview}.

\begin{figure*}[t!]
\centering
\includegraphics[width=\textwidth]{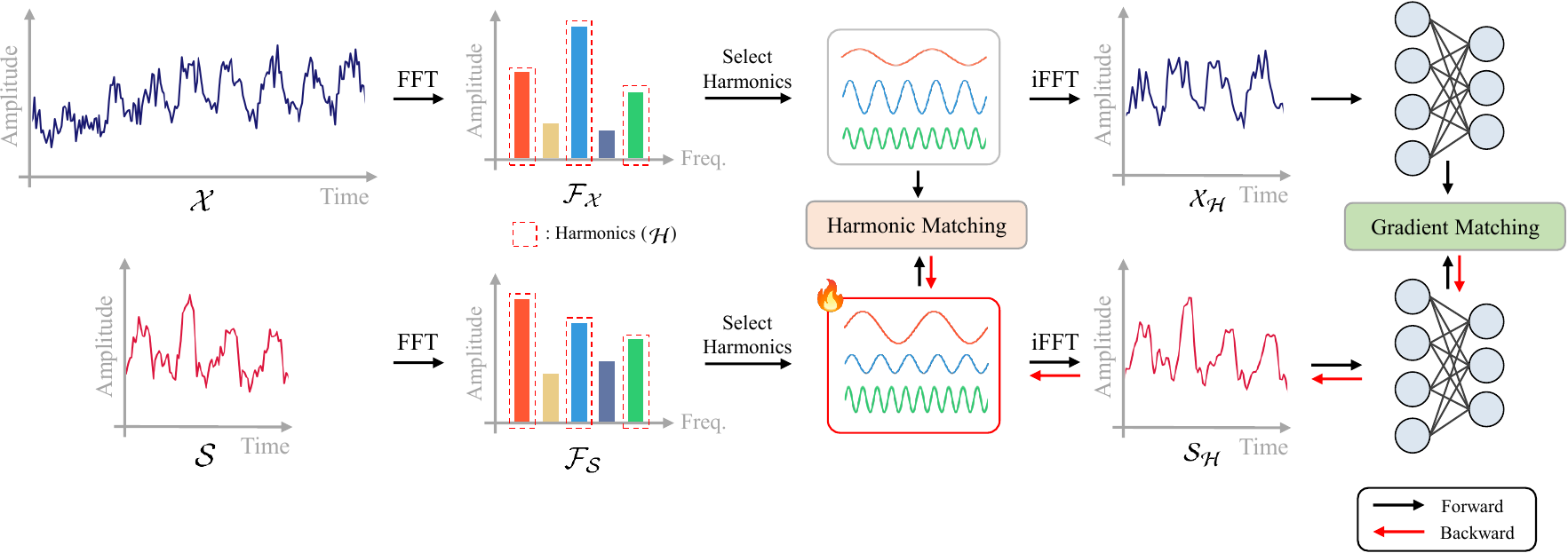}
\caption{An overview of our method. Selected harmonics (red box) of synthetic data are updated through \emph{Harmonic Matching} and \emph{Gradient Matching}.
}
\label{overview}
\end{figure*}

\subsection{Harmonic Matching}

The primary objective of Harmonic Matching is to align the harmonic distributions of the synthetic data $\mathcal{S}$ and the original data $\mathcal{X}$. To achieve this, we begin by decomposing the time series into a sum of sinusoidal basis functions with FFT.

A precise alignment of basis functions between $\mathcal{X}$ and $\mathcal{S}$ is required, but their different lengths prevent a direct one-to-one matching of components that represent the \emph{same period}. Furthermore, a longer sequence produces a high-resolution spectrum in which numerous high-frequency components can act as noise, diluting the prominence of harmonics. To resolve these issues, we sample a subsequence $\mathcal{X}_\text{sub}$ from $\mathcal{X}$ with the same length $M$ as $\mathcal{S}$ and apply the FFT to both to obtain their frequency domain representations:

\begin{equation}
\begin{aligned}
\mathcal{F_{X}} = \texttt{FFT}(\mathcal{X}_{\text{sub}}),& \quad \mathcal{F_S} = \texttt{FFT}(\mathcal{S}). 
\end{aligned}
\end{equation}
Here, the FFT efficiently computes the Discrete Fourier Transform (DFT), which decomposes the sequences $\mathcal{X}_{sub} = \{x_n\}^{M-1}_{n=0}$ and $\mathcal{S} = \{s_n\}^{M-1}_{n=0}$ of length $M$ into its frequency components $\mathcal{F_X}$ and $\mathcal{F_S}$ as follows:
\begin{equation}
\begin{aligned}
\mathcal{F_X}[k] &= \sum_{n=0}^{M-1} x_n e^{-i 2\pi k n / M}, \quad \text{and}\\
\mathcal{F_S}[k] &= \sum_{n=0}^{M-1} s_n e^{-i 2\pi k n / M}, \quad k=0, \dots, M-1.
\label{eq:dft}
\end{aligned}
\end{equation}

Using all these frequency components is often suboptimal because it incorporates unnecessary noise, which can obscure the true underlying patterns. Therefore, we selectively distill the most significant basis functions, which correspond to the top-k frequencies with the largest amplitudes in $\mathcal{F_X}$.
We refer to these top-k components as \emph{harmonics} $\mathcal{H}$, and obtain the following harmonics-based sequence representations:
\begin{equation}
\begin{aligned}
        \mathcal{H} &= \text{arg} \: \text{top-}k_{i\in[0, \lfloor M/2\rfloor ]} \: (|\mathcal{F_X}[i]|) \\
\tilde{\mathcal{F_X}}[i] :&= 
\begin{cases}
\mathcal{F_X}[i], & i \in \mathcal{H} \\
0, & \text{otherwise}
\end{cases} \quad \forall i \in [0, \lfloor M/2\rfloor], \\
\tilde{\mathcal{F_S}}[i] :&= 
\begin{cases}
\mathcal{F_S}[i], & i \in \mathcal{H} \\
0, & \text{otherwise}
\end{cases} \quad \forall i \in [0, \lfloor M/2\rfloor].
\end{aligned}
\label{eq:fx}
\end{equation}
During the distillation process, we directly update these selected harmonics instead of data points within local windows. Since each harmonic is a sinusoidal basis function with a global influence, every update modifies the synthetic series as a whole.

This global update mechanism is the key to resolving the scalability limitation (\textbf{L1}). The length of the synthetic series determines the range of periods that the sinusoidal basis can represent. Increasing $M$ allows our method to capture the harmonics corresponding to \emph{longer periods}, which contain a richer set of long-range global structures. This provides a key advantage over window-based methods, which are restricted to learning additional local patterns without capturing a broader context. This ensures that the performance of our method scales meaningfully with the size of the synthetic dataset.

To enforce similarity between the selected harmonics, we introduce a harmonic loss $\mathcal{L}_\text{harm}$, which minimizes the Lp-norm distance between the amplitudes of these harmonic coefficients:
\begin{equation}
    \mathcal{L}_\text{harm} = \| \: | \tilde{\mathcal{F_X}}| -  |\tilde{\mathcal{F_S}}|\:\|_p.
\label{eq:freq_loss}
\end{equation}
Minimizing this loss acts as a regularizer, forcing the periodic structure of $\mathcal{S}$ to align with that of $\mathcal{X}$. This resolves the architectural overfitting problem \textbf{(L2)}, since the synthetic data learns to capture the harmonic distribution, which is an intrinsic model-agnostic property of the data, rather than overfitting to the specific biases of a single backbone. The formal justification for how minimizing $\mathcal{L}_\text{harm}$ preserves the global structure of the original series is provided in Theorem~\ref{thm:minimize_FX_FS}.

\newcommand{\Var}{\mathrm{Var}}
\newcommand{\Cov}{\mathrm{Cov}}
\newcommand{\E}{\mathbb{E}}
\newcommand{\psd}{{PSD}}
\newcommand{\ifft}{\texttt{iFFT}}
\newcommand{\fft}{\texttt{FFT}}

\subsubsection{Theoretical Analysis for Harmonic Matching}

We further provide a theoretical justification for how \emph{Harmonic Matching} successfully preserves the global structure of time series data. Our analysis is grounded in the relationship between a series's Power Spectral Density (PSD), which describes its power distribution across frequencies, and its Autocorrelation Function (ACF), which measures its temporal dependencies. The core principle is that by aligning the most significant components of the PSD, the \emph{harmonics}, we ensure that the autocorrelation structures of the original and synthetic data are also aligned.

\newtheorem{theorem}{Theorem}

\begin{theorem}
\label{thm:minimize_FX_FS}
\quad
 Let $\mathcal{F_X}$
  and $\mathcal{F_S}$ are the DFTs of an $M$-point 
subset (segment) of $\mathcal{X}$ and $\mathcal{S}$, and let
  $r_\mathcal{X}(k)$ and $r_\mathcal{S}(k)$ denote their respective ACFs at lag
  $k$. Suppose we choose $\mathcal{S}$ so as to minimize $\|\:|\mathcal{F_X}| - |\mathcal{F_S}|\:\|_p$. 
Then, for a given maximum lag $K$, there exists a constant $C>0$ such that the difference in their ACFs is bounded:
\[
\max_{|k|\le K}
\bigl|r_\mathcal{S}(k) - r_\mathcal{X}(k)\bigr|
\;\le\;
C\,\varepsilon,
\]
where $\varepsilon$ is a measure of how closely $\mathcal{F_S}$ approximates $\mathcal{F_X}$ 
in the frequency domain.
\end{theorem}

Theorem 1 provides the theoretical foundation for Harmonic Matching. The inequality shows that the maximum error between the autocorrelation of the original and synthetic datasets is directly controlled by $\varepsilon$, the frequency domain approximation error. Since Harmonic Matching is explicitly designed to minimize this error, the theorem guarantees that it effectively preserves the underlying temporal dependencies of the original time series. For a detailed derivation and the necessary technical assumptions, refer to the full proof in Appendix. 

\subsection{Gradient Matching} \label{gradient_matching}
For use in the subsequent gradient matching step, we reconstruct time-domain signals containing only these essential harmonic patterns via the inverse FFT (iFFT):
\begin{equation}
       \mathcal{X_H} = \texttt{iFFT}( \tilde{\mathcal{F_X}}), \quad \mathcal{S_H} = \texttt{iFFT}( \tilde{\mathcal{F_S}}).
\label{eq:recon}
\end{equation}

We employ the surrogate objective similar to previous works \cite{cazenavette2022dataset, cui2023scaling}, which matches the multi-step gradients with respect to the model parameters $\theta$ when the model is trained on $\mathcal{S}$ versus on $\mathcal{X}$. We train the model parameter $\theta$ on $\mathcal{X_H}$ for $i$ steps and on $\mathcal{S_H}$ for $j$ steps, then use their normalized Euclidean distance as the distillation loss:
\begin{equation}
    \mathcal{L}_{\text{grad}} = \frac{||\mathcal{T}_j(\theta, \mathcal{S_H})-\mathcal{T}_i(\theta, \mathcal{X_H})||^2_2}{||\theta-\mathcal{T}_i(\theta, \mathcal{X_H})||^2_2}.
\label{grad_loss}
\end{equation}
Here, $\mathcal{T}_i$ denotes the $i$-step training process.

By combining $\mathcal{L}_\text{grad}$ with $\mathcal{L}_\text{harm}$, we arrive at the final objective function:
\begin{equation}
    \underset{\mathcal{F_S}}{\text{argmin}} \: \mathcal{L}_\text{grad} + \lambda \mathcal{L}_\text{harm},
\label{eq:final_loss}
\end{equation}
where $\lambda$ serves as a hyperparameter to balance the contributions of the two losses. Once $\mathcal{F_S}$ has converged, the final distilled dataset $\mathcal{S}$ is recovered with the iFFT. The entire distillation process is described in Appendix.

\input{total_table}

\section{Experiments}
\subsection{Setup}

\subsubsection{\textbf{Datasets}}
We evaluated our method on widely used TSF benchmarks: the ETT, Electricity and Traffic datasets \cite{lstnet, informer}. The details of each dataset are provided in Appendix.

\subsubsection{\textbf{Backbones}}
Previous state-of-the-art methods \cite{ding2024condtsf} 
have employed simple models (e.g., MLP, LSTM, CNN) as backbones. To evaluate the
generality of the dataset distillation algorithm, we instead
use recent state-of-the-art TSF models: DLinear \cite{zeng2023transformers},
iTransformer \cite{itransformer}, and xPatch \cite{xpatch}, which are the representatives of Linear, Transformer, and
CNN architectures, respectively. The details of
the models are in Appendix.

\subsubsection{\textbf{Baselines}}
We compare our method against several baselines, starting with a naive `Random' that trains on a randomly selected subsequence of the original data. We then include state-of-the-art dataset distillation methods: DC~\cite{zhao2021datasetcondensation}, MTT~\cite{cazenavette2022dataset}, TESLA~\cite{cui2023scaling}, and CondTSF~\cite{ding2024condtsf}. The performance of `Full Data' training is also reported as a lower bound of MSE. Detailed descriptions of each baseline are provided in the Appendix.

\begin{figure*}[t]
  \centering
  \includegraphics[width=\textwidth]{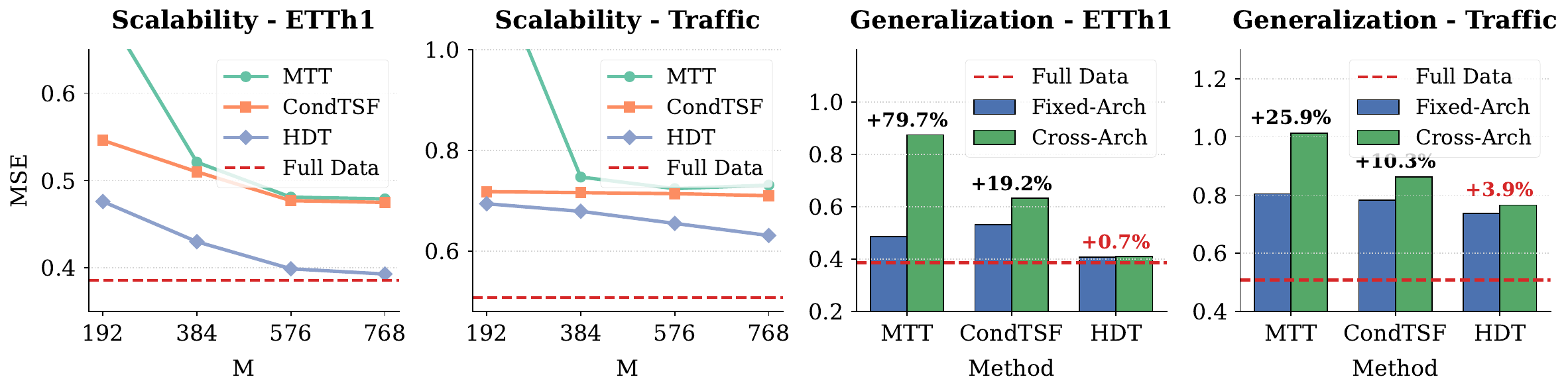}
  \caption{Scalability and cross-architecture generalization performance on the ETTh1 and Traffic datasets. The left two plots show scalability results for varying synthetic data sizes ($M$). The right two plots compare the performance between fixed- and cross-architecture settings, highlighting that HDT maintains a significantly smaller increase in MSE compared to others.}
\label{scalability}
\end{figure*}

\subsection{Overall Results} 
We conducted comprehensive experiments to validate the empirical performance of HDT, following the evaluation protocol described in Appendix.
The evaluation covers both \emph{fixed-architecture} settings, where the backbone and evaluation models are the same, and \emph{cross-architecture} settings, where they differ. Table~\ref{table:performance} presents the overall results for $M=384$, using DLinear (L), iTransformer (T), and xPatch (C) as the backbone and the evaluation models.

A key finding is that while prior methods perform well in fixed-architecture settings, they suffer from unreliable performance in cross-architecture scenarios. Their performance often degrades significantly when the backbone and evaluation models differ, in some cases performing even worse than the `Random' baseline (e.g., xPatch backbone evaluated on DLinear). This highlights their severe architectural overfitting (\textbf{L2}), as the distilled data is over-specialized to a single model.

In contrast, HDT not only achieves state-of-the-art performance across nearly all settings but also demonstrates remarkable robustness. Its performance remains strong and stable in cross-architecture settings, indicating that it distills a more universal, model-agnostic representation of the original data. The two rightmost plots of Figure~\ref{scalability} (derived from Table~\ref{table:performance}) clearly illustrate that only HDT maintains a minimal increase in MSE in cross-architecture settings.

Finally, to evaluate the scalability of our method (\textbf{L1}), we measured the performance as the synthetic data size $M$ increases. The two leftmost plots in Figure~\ref{scalability} show that HDT's performance consistently improves with a larger $M$, while others saturate beyond a certain size. This confirms HDT's ability to effectively capture the long-range context of the time series.

To supplement our main findings, we provide additional results including a detailed hyperparameter analysis, experiments with a smaller synthetic dataset, and qualitative visualizations in Appendix. 

\subsection{Ablation Study}
To validate the contribution of each key component in HDT, we conducted an ablation study using the DLinear backbone. We compared three distinct configurations, measuring their average performance across all evaluation models (L, T, and C): (1) Base, which uses only conventional gradient matching on time windows; (2) Base + Decomp., which performs gradient matching only with sinusoidal decomposition; and (3) HDT, the full proposed method.

The results in Table~\ref{ablation} show the effectiveness of our components. The improvement from `Base' to `Base + Decomp.' demonstrates the advantage of operating in the frequency domain, as each update has a global influence over the entire sequence. The final substantial improvement of `HDT' is delivered by Harmonic Matching, which distills the essential periodic patterns by aligning the distribution of selected harmonics between the synthetic and original data.

\begin{table}[t]
    \centering
    \renewcommand{\arraystretch}{1.1}
    \resizebox{\columnwidth}{!}{
    \begin{tabular}{l|cccccc}
        \toprule
        Method & ETTh1 & ETTh2 & ETTm1 & ETTm2 & Electricity & Traffic\\ 
        \midrule
        Base           & 0.583 & 0.465 & 0.905 & 0.402 & 0.414 & 0.934 \\
        Base + Decomp.  & 0.545 & 0.420 & 0.814 & 0.325 & 0.376 & 0.902\\
        HDT & \textbf{0.420} & \textbf{0.334} & \textbf{0.386} & \textbf{0.206} & \textbf{0.226} & \textbf{0.760}  \\
        \bottomrule

    \end{tabular}
    }
    \caption{Ablation study on the components of HDT. Base denotes conventional gradient matching, and Decomp. refers to the use of the sinusoidal decomposition.}  
\label{ablation}
    \renewcommand{\arraystretch}{1.0}
\end{table}

\begin{table}
    \centering
    \footnotesize
    \renewcommand{\arraystretch}{0.9}
    \begin{tabularx}{\columnwidth}{ll S S S}
        \toprule
        Backbone & Dataset & {MTT} & {CondTSF} & {HDT} \\
        \midrule
        \multirow{2}{*}{DLinear} & ETTh1 & 30.06 & 32.13 & 36.37 \\
        & Electricity & 35.61 & 38.04 & 42.76 \\
        \midrule
        \multirow{2}{*}{iTransformer} & ETTh1 & 197.69 & 199.95 & 209.46 \\
        & Electricity & 253.90 & 262.10 & 264.90 \\
        \bottomrule
    \end{tabularx}
    \caption{Comparison of running times over 100 outer-loops in seconds. The cost increase of HDT is marginal.}
    \label{tab:runtime}
\end{table}

\newcolumntype{C}{>{\centering\arraybackslash}X}

\subsection{Efficiency Analysis}

\subsubsection{Distillation Runtime}
We report the distillation runtime (in seconds) over 100 outer-loop steps in Table~\ref{tab:runtime}. Although HDT incorporates the FFT, its theoretical complexity of \(O(M \log M)\) is minor compared to the gradient computations of the backbone model. As Table~\ref{tab:runtime} shows, the runtime gap between HDT and other methods narrows as model complexity increases (from DLinear to iTransformer), confirming that the overall overhead of HDT is marginal. These results, combined with the performance gains reported in Table~\ref{table:performance}, demonstrate that HDT achieves superior performance with only a slight increase in computational cost.

\subsubsection{Training Efficiency}
We measured the actual training time of iTransformer on the full dataset versus our distilled dataset. As shown in Table~\ref{tab:traintime}, training on the distilled data achieves a dramatic speed-up, reducing training time from hours to mere seconds. This highlights the substantial computational benefits of the distilled dataset.

\subsection{Large Scale Scenario}\label{largescale}

\subsubsection{Performance on Large-Scale Datasets}
Evaluations in TSF often rely on lightweight datasets that may not represent the scale of real-world data. To address this gap, we evaluated HDT on the massive CA dataset from the LargeST \cite{largest} benchmark, which has a length of 201,363 and 8,600 features. For the experiment, we distilled this dataset to a size of $M=384$ using the DLinear backbone. As shown in Table~\ref{tab:large_data}, HDT scales effectively to massive data, significantly outperforming previous distillation methods and approaching the performance of full-data training.

\subsubsection{Fine-tuning Large Foundation Models}
We further investigated HDT's utility in the context of large foundation models. We evaluated the pretrained Moirai-Large (311M parameters) \cite{moirai} on the ETTh1 dataset under three scenarios: zero-shot, full fine-tuning, and few-shot fine-tuning with a dataset distilled with HDT. Table~\ref{tab:large_model} shows that the distilled dataset enables substantial performance gains over zero-shot with a minimal fraction of the training cost of full fine-tuning, demonstrating a highly favorable cost-performance trade-off. Specifically, our approach is 80x faster than full fine-tuning while incurring only a 2.5\% degradation in MSE performance. Even including the distillation process (end-to-end), HDT still achieves an 8.51x speed-up.

\begin{table}[t]
    \centering
    \footnotesize
    \renewcommand{\arraystretch}{0.9}
    \begin{tabularx}{\columnwidth}{l C C C}
        \toprule
        Dataset & {Full Training} & {Distilled} & {Speed-up} \\
        \midrule
        Electricity & 1650.33 & 1.98 & \textbf{834x} \\
        Traffic & 4266.45 & 2.32 & \textbf{1839x} \\
        \bottomrule
    \end{tabularx}
    \caption{Training time comparison in seconds with the iTransformer backbone, demonstrating the significant efficiency gains achieved by training on the distilled dataset.}
    \label{tab:traintime}
\end{table}

\begin{table}[t]
    \footnotesize
    \centering
    \renewcommand{\arraystretch}{0.9}
    \begin{tabularx}{\columnwidth}{lCCCC}
        \toprule
         & Random & CondTSF & HDT & Full Data \\
        \midrule
        MSE & 358.51 & 197.95 & \textbf{46.63} & 44.25 \\
        \bottomrule
    \end{tabularx}
    \caption {Distillation performance (MSE) on CA dataset with the DLinear backbone. HDT's performance closely approaches that of full-data training.}
    \label{tab:large_data} 
\end{table}

\begin{table}[t]
    \centering
    \footnotesize
    \renewcommand{\arraystretch}{0.9}
    \begin{tabularx}{\columnwidth}{lcc}
        \toprule
        Setting & {MSE} & {Training Time} \\
        \midrule
        Zero-shot & 1.972 & {-} \\
        \midrule
        Full Fine-tuning & 1.383 & 3775.68s \\
        Few-shot with HDT (end-to-end) & 1.417 & 443.51s \\
        Few-shot with HDT & 1.417 & 47.23s \\
       \midrule
        Comparison & {\textbf{+2.5\%}} & {\textbf{80x (8.51x)}} \\ 

        \bottomrule
    \end{tabularx}
    \caption{Fine-tuning the Moirai-Large (311M) on ETTh1. 
    The final row shows the MSE increase and training speed-up of our method relative to full fine-tuning.}
    \label{tab:large_model}
\end{table}

\section{Related Works}
\subsection{Time Series Forecasting}

Recent research in Time Series Forecasting (TSF) has evolved along several distinct architectural paradigms. Transformer-based models are prominent for their ability to capture long-range dependencies \cite{itransformer, nie2023time}, simple Linear-based models have demonstrated surprisingly strong performance \cite{zeng2023transformers, ekambaram2023tsmixer, xu2024fits, timemixer}, and CNN-based approaches have proven effective for capturing local temporal patterns \cite{wu2022timesnet, donghao2024moderntcn, xpatch}. A more recent trend is the development of large pre-trained foundation models, which have shown zero-shot forecasting capabilities \cite{timesfm, moirai}. To demonstrate the general applicability of HDT across these distinct paradigms, we select representative models from each: iTransformer \cite{itransformer}, DLinear \cite{zeng2023transformers}, xPatch \cite{xpatch}, and Moirai \cite{moirai}.

\subsection{Dataset Distillation}
Dataset Distillation (DD) aims to create a compact synthetic dataset that yields training performance comparable to that of the original dataset. This concept was first introduced in \cite{wang2018datasetdistillation}, but it suffered from suboptimal performance due to the need for nested-loop optimization. To address this issue, three main approaches have been developed: (1) optimizing surrogate objective, (2) kernel-based methods, and (3) parameterization. Instead of directly minimizing the final performance, several works optimize surrogate objectives such as matching gradients \cite{zhao2021datasetcondensation}, training trajectories \cite{cazenavette2022dataset, cui2023scaling, guo2024datm}, or feature distributions \cite{zhao2023distribution, zhang2024m3d} between the synthetic and original datasets. Kernel-based methods simplify the optimization process with kernel ridge regression to derive a closed-form solution for the inner loop \cite{nguyen2021kip, nguyen2021dataset, zhou2022dataset}. Parameterization methods optimize alternative parameters instead of directly optimizing image pixels \cite{cazenavette2023generalizing, liu2023mgdd, liu2023fewshot, wei2023sparse, Fred}. Our work combines a surrogate objective and a parameterization method, both tailored for the unique challenges of the TSF dataset distillation.

\section{Conclusion}

In this work, we introduced HDT, a novel dataset distillation paradigm for TSF. By shifting the distillation process to the frequency domain to align the \emph{harmonics}, HDT effectively preserves the global structure of the original series. Our experiments demonstrate that this approach resolves the critical limitations of prior methods: limited scalability and architectural overfitting. Furthermore, we validated HDT's practical utility in large-scale, real-world scenarios, including its performance on massive datasets and its effectiveness in fine-tuning foundation models. As the need for efficient data handling becomes increasingly critical in domains such as online learning and resource-constrained settings, the application of DD for TSF presents a promising avenue for future research.

\section*{Acknowledgements}
This work was supported by the National Research Foundation of Korea (NRF) grant funded by the Korea government (MSIT) (No. RS-2023-00217286, No. RS-2024-00335873), Institute of Information \& communications Technology Planning \& Evaluation (IITP) grant funded by the Korea government(MSIT) (No.RS-2019-II191906, Artificial Intelligence Graduate School Program(POSTECH)), and the Technology Innovation Program (No. RS-2025-02952974) funded by the Ministry of Trade, Industry \& Energy (MOTIE, Korea).

\bibliography{aaai2026}

\setcounter{secnumdepth}{2}

\newpage 
\clearpage 

\appendix
\renewcommand\thesubsection{\thesection.\arabic{subsection}}

\section{Method Details}
\subsection{Proof of Theorem 1.}\label{theoremproof}
In this section, we provide the full proof of Theorem \ref{thm:minimize_FX_FS}. Here are the preliminaries and notations:

\begin{itemize}[leftmargin=15pt]
  \item Assume that both time series \( \mathcal{X} \)\footnote{For simplicity, throughout this proof, we use $\mathcal{X}$ to denote $\mathcal{X}_\text{sub}$.} and \( \mathcal{S} \) follow an autoregressive process \cite{autoreg}. 
  \item \(\omega_j\) denotes a discrete set of frequencies where we approximate or evaluate the PSD.
  \item \(r_X(k)\) and \(r_S(k)\) represent the ACF at lag \(k\) for \(\mathcal{X}\) and \(\mathcal{S}\), respectively.
  \item \(\mathcal{F_X}\) and \(\mathcal{F_S}\) are the discrete Fourier transforms (DFTs) of \(\mathcal{X}\) and \(\mathcal{S}\), respectively.
  \item We assume \(\psd_\mathcal{X}(\omega)\) and \(\psd_\mathcal{S}(\omega)\) are well-defined, with \
  \(\psd_\mathcal{X}(\omega)\) corresponding to a stationary or stable AR process \cite{autoreg}.
\end{itemize}
\begin{proof}[Proof]

We outline the key steps:
\paragraph{Step 1: Linking DFT distance to PSD distance.}
Let $\mathcal{F_X}(\omega_j)$ and $\mathcal{F_S}(\omega_j)$ denote the DFT coefficients 
of $\mathcal{X}$ and $\mathcal{S}$ at frequency
$\omega_j=\frac{2\pi j}M$, for $j=0,\dots,M-1$.  
By assumption, we pick $\mathcal{S}$ so that $\|\:|\mathcal{F_X}| - |\mathcal{F_S}|\:\|_p$ is minimized. 
In particular, each component difference 
$\bigl||\mathcal{F_X}(\omega_j)| - |\mathcal{F_S}(\omega_j)|\bigr|$ 
is controlled by a small quantity (call it $\delta$).  
Since the periodogram (or empirical PSD) at $\omega_j$ is given by
\[
\widehat{f_\mathcal{X}}(\omega_j)
\;=\;
\frac{1}{M}\,\bigl|\mathcal{F_X}(\omega_j)\bigr|^2,
\quad
\widehat{f_\mathcal{S}}(\omega_j)
\;=\;
\frac{1}{M}\,\bigl|\mathcal{F_S}(\omega_j)\bigr|^2,
\]

the difference in PSDs at $\omega_j$ can be bounded as 

\[
\begin{aligned}
&\bigl|\widehat{f_\mathcal{X}}(\omega_j) - \widehat{f_\mathcal{S}}(\omega_j)\bigr|\\
\;=\;&
\frac{1}{M}
\Bigl|\,
\bigl|\mathcal{F_X}(\omega_j)\bigr|^2 - \bigl|\mathcal{F_S}(\omega_j)\bigr|^2
\Bigr| \\
\;\le\;&
\frac{1}{M}\,\bigl||\mathcal{F_X}(\omega_j)| - |\mathcal{F_S}(\omega_j)|\bigr|\,
\bigl(\lvert \mathcal{F_X}(\omega_j)\rvert + \lvert \mathcal{F_S}(\omega_j)\rvert\bigr).
\end{aligned}
\]

Hence, whenever 
$\bigl||\mathcal{F_X}(\omega_j)| - |\mathcal{F_S}(\omega_j)|\bigr| \le \delta$ and the magnitudes $\lvert \mathcal{F_X}(\omega_j)\rvert$, 
$\lvert \mathcal{F_S}(\omega_j)\rvert$ remain in a bounded set 
(say each $\le K$), it follows that
\[
\bigl|\widehat{f_\mathcal{X}}(\omega_j) - \widehat{f_\mathcal{S}}(\omega_j)\bigr|
\;\le\;
\frac{K}{M}\,\delta.
\]
We incorporate constants into a single parameter $\varepsilon$, 
so that
\[
\bigl|\widehat{f_\mathcal{X}}(\omega_j) - \widehat{f_\mathcal{S}}(\omega_j)\bigr|
\;\le\;
\varepsilon,
\quad \forall\, \omega_j.
\]
That is, the empirical PSD of $\mathcal{S}$ closely matches that of $\mathcal{X}$ 
at the sampled frequencies.

\paragraph{Step 2: Extending PSD proximity to the entire spectrum.}
Under the assumption that both $\psd_\mathcal{X}(\omega)$ and $\psd_\mathcal{S}(\omega)$ 
are sufficiently smooth (e.g., Lipschitz) on $[-\pi,\pi]$ \cite{kullback}, we can leverage an interpolation/continuity argument as follows. Assume there exists a discrete 
frequency set $\{\omega_j\}_{j=0}^{N-1}$ such that
\[
\bigl|\psd_\mathcal{S}(\omega_j) - \psd_\mathcal{X}(\omega_j)\bigr|
\;\le\;
\varepsilon,
\quad 
\forall\, \omega_j.
\]
Denote the Lipschitz constants of $\psd_\mathcal{X}$ and $\psd_\mathcal{S}$ 
by $\ell_\mathcal{X}$ and $\ell_\mathcal{S}$, respectively.  
Then, for any $\omega \in [-\pi,\pi]$, pick $\omega_j$ nearest to $\omega$ 
so that $|\omega - \omega_j|\le \Delta$, 
where $\Delta$ is half the grid spacing in frequency.  
By the Lipschitz property,
\[
\begin{aligned}
\bigl|\psd_\mathcal{X}(\omega) - \psd_\mathcal{X}(\omega_j)\bigr|
\;&\le\;
\ell_\mathcal{X}\,\bigl|\omega - \omega_j\bigr|,
\quad\\
\bigl|\psd_\mathcal{S}(\omega) - \psd_\mathcal{S}(\omega_j)\bigr|
\;&\le\;
\ell_\mathcal{S}\,\bigl|\omega - \omega_j\bigr|.
\end{aligned}
\]
Hence,
\[
\begin{aligned}
\bigl|\psd_\mathcal{S}(\omega) - \psd_\mathcal{X}(\omega)\bigr|
\;\le&\;
\bigl|\psd_\mathcal{S}(\omega) - \psd_\mathcal{S}(\omega_j)\bigr|\\
&+
\bigl|\psd_\mathcal{S}(\omega_j) - \psd_\mathcal{X}(\omega_j)\bigr|\\
&+
\bigl|\psd_\mathcal{X}(\omega_j) - \psd_\mathcal{X}(\omega)\bigr|
\\[6pt]
\;\le&\;
\ell_\mathcal{S}\,\Delta
\;+\;
\varepsilon
\;+\;
\ell_\mathcal{X}\,\Delta.
\end{aligned}
\]

Define
\[
C
\;=\;
1 \;+\; \ell_\mathcal{S}\,\Delta / \varepsilon
\;+\;
\ell_\mathcal{X}\,\Delta / \varepsilon
,\;\;\;
\]
so that
\[
\bigl|\psd_\mathcal{S}(\omega) - \psd_\mathcal{X}(\omega)\bigr|
\;\le\;
C\,\varepsilon
\quad
\text{for all }\,\omega\in[-\pi,\pi].
\]
In other words, if $\psd_\mathcal{S}(\omega_j)$ is $\varepsilon$-close 
to $\psd_\mathcal{X}(\omega_j)$ at the discrete $\{\omega_j\}$, 
then $\psd_\mathcal{S}(\omega)$ remains on the order of $\varepsilon$-close 
to $\psd_\mathcal{X}(\omega)$ on the entire interval.

\paragraph{Step 3: Translating PSD closeness into ACF closeness.}
By the Wiener--Khinchin theorem \cite{wiener, khintchine}, the ACFs 
$r_\mathcal{X}(k)$ and $r_\mathcal{S}(k)$ are related to the PSDs via
\[
\begin{aligned}
r_\mathcal{X}(k)
&\;=\;
\frac{1}{2\pi}\int_{-\pi}^{\pi}\psd_\mathcal{X}(\omega)\,e^{\,i\,\omega k} \, d\omega,\\
\quad
r_\mathcal{S}(k)
&\;=\;
\frac{1}{2\pi}\int_{-\pi}^{\pi}\psd_\mathcal{S}(\omega)\,e^{\,i\,\omega k} \, d\omega.
\end{aligned}
\]
Thus,
\[
r_\mathcal{S}(k) - r_\mathcal{X}(k)
\;=\;
\frac{1}{2\pi} \int_{-\pi}^{\pi} 
\bigl[\psd_\mathcal{S}(\omega)-\psd_\mathcal{X}(\omega)\bigr]\,e^{\,i\,\omega k} \, d\omega.
\]
Taking absolute values and using $|e^{\,i\,\omega k}|=1$,
\[
\bigl|r_\mathcal{S}(k) - r_\mathcal{X}(k)\bigr|
\;\le\;
\frac{1}{2\pi}\int_{-\pi}^{\pi}
\bigl|\psd_\mathcal{S}(\omega) - \psd_\mathcal{X}(\omega)\bigr|\,
d\omega.
\]
Since $\bigl|\psd_\mathcal{S}(\omega)-\psd_\mathcal{X}(\omega)\bigr|\le C\,\varepsilon$ 
for some constant $C$ (by the continuity argument in Step 2), 
we have
\[
\bigl|r_\mathcal{S}(k) - r_\mathcal{X}(k)\bigr|
\;\le\;
\frac{1}{2\pi}\int_{-\pi}^{\pi}
C\,\varepsilon \, d\omega
\;=\;
C\,\varepsilon.
\]
Hence, for each integer lag $k$,
\[
\bigl|r_\mathcal{S}(k) - r_\mathcal{X}(k)\bigr|
\;\le\;
C\,\varepsilon.
\]

\paragraph{Step 4: Bounding the maximum lag.}
Restricting to lags $\lvert k\rvert \le K$, we trivially obtain
\[
\max_{\lvert k\rvert \,\le\, K}
\bigl|r_\mathcal{S}(k) - r_\mathcal{X}(k)\bigr|
\;\le\;
C\,\varepsilon.
\]

This completes the proof of Theorem~\ref{thm:minimize_FX_FS}.
\end{proof}

\subsection{Algorithm}\label{app_algorithm}
We provide the overall process of HDT in Algorithm 1. The algorithm begins by initializing a random synthetic dataset $\mathcal{S}$ and converting it to its frequency-domain representation, $\mathcal{F_S}$. In each iteration, a subsequence $\mathcal{X}_\text{sub}$ is sampled from the original data, and transformed to $\mathcal{F_X}$ with FFT. The harmonics are then extracted from the spectra to obtain the filtered representations $\tilde{\mathcal{F_X}}$ and $\tilde{\mathcal{F_S}}$, which are used to calculate the harmonic loss $\mathcal{L}_\text{harm}$. These filtered components are also reconstructed into time-domain signals $\mathcal{X_H}$ and $\mathcal{S_H}$ through the inverse FFT to compute the gradient matching loss $\mathcal{L}_\text{grad}$. Combining $\mathcal{L}_\text{harm}$ and $\mathcal{L}_\text{grad}$, the synthetic representation $\mathcal{F_S}$ is updated. After convergence, the final $\mathcal{F_S}$ is converted back to the time domain with an inverse FFT to yield the distilled dataset $\mathcal{S}$.

\section{Experimental Details}

\subsection{Datasets} \label{app_datasets}
We evaluated our method on widely-used public benchmarks for TSF. Detailed statistics for each dataset are provided in Table~\ref{table:stat}, and a brief description of each follows:
\begin{itemize}[leftmargin=15pt]
\item\textbf{ETT}\footnote{https://github.com/zhouhaoyi/ETDataset} consists of oil temperature and six power load features recorded at hourly and 15-minute intervals, collected from electricity transformers from 2016 to 2018.
\item\textbf{Electricity}\footnote{\url{https://archive.ics.uci.edu/dataset/321/electricityloaddiagrams20112014}} records the hourly electricity consumption of 321 clients from 2012 to 2014.
\item\textbf{Traffic}\footnote{https://pems.dot.ca.gov/} provides hourly road occupancy rates measured by sensors on San Francisco Bay area freeways from 2015 to 2016.
\item\textbf{CA}\footnote{https://github.com/liuxu77/LargeST} provides traffic flow data from 8,600 mainline sensors across California, recorded at 5-minute intervals from 2017 to 2021. 

\end{itemize}

\begin{table}[h]
    \renewcommand{\arraystretch}{0.8}
   \centering
   \resizebox{\columnwidth}{!}{
    \begin{tabular}{r|ccccccc}
    \toprule
    Dataset & ETTh1 & ETTh2 & ETTm1 & ETTm2 & Electricity & Traffic & CA \\
    \midrule
    Length    & 17,420 & 17,420 & 69,680 & 69,680  & 26,304 & 17,544 & 201,363  \\
    Features  & 7     & 7     & 7     & 7   & 321 & 862 & 8,600\\ 
    \bottomrule
    \end{tabular}
    }
   \caption{Statistics of the datasets used for evaluation.}\label{table:stat}
   \renewcommand{\arraystretch}{1.0}
\end{table}
\subsection{Backbones} \label{app_backbones}
We provide a detailed description of the backbone models used in our experiments:
\begin{itemize}[leftmargin=15pt]
    \item \textbf{DLinear} \cite{zeng2023transformers} is a simple yet effective MLP-based model for time series forecasting, which decomposes the time series into seasonality and trend and performs linear regression on each component.
    \item \textbf{iTransformer} \cite{itransformer} inverts the Transformer architecture by applying self-attention to variate tokens and feed-forward networks to temporal sequences.
    \item \textbf{xPatch} \cite{xpatch} is a dual-stream time series forecasting model that integrates CNN and MLP architectures with Exponential Moving Average (EMA) decomposition, efficiently capturing trend and seasonal components.
    \item \textbf{Moirai} \cite{moirai} is a foundation model for TSF, built on a Transformer architecture and pre-trained on the large-scale dataset to perform zero-shot predictions across diverse data with varying frequencies, dimensions, and distributions.
\end{itemize}

\begin{algorithm}[t]
\caption{HDT}\label{FD}
\raggedright
\textbf{Input:} Training dataset $\mathcal{X}$, hyper-parameters $k$ and $\lambda$, learning rates $\eta$\\
\textbf{Output:} Distilled dataset $\mathcal{S}$ 
\begin{algorithmic}[1]
\State Initialize $\mathcal{S}$.
\State $\mathcal{F_S} = \texttt{FFT}(\mathcal{S})$
\While{not converged}
    \State Initialize the model parameter $\theta.$
    \State Sample a subsequence $\mathcal{X}_\text{sub}$ from $\mathcal{X}$.
    \State $\mathcal{F_X} = \texttt{FFT}(\mathcal{X}_\text{sub})$
    \State Obtain $ \tilde{\mathcal{F_X}},   \tilde{\mathcal{F_S}}$. \Comment{Eq~\ref{eq:fx}.}
    \State Calculate $\mathcal{L}_\text{harm}$. \Comment{Eq~\ref{eq:freq_loss}.}
    \State Reconstruct $\mathcal{X_H}$ and $\mathcal{S_H}$. \Comment{Eq~\ref{eq:recon}.}
    \State Calculate $\mathcal{L}_\text{grad}$. \Comment{Eq~\ref{grad_loss}.}
    \State Update $\mathcal{F_S} \leftarrow \mathcal{F_S} - \eta \nabla (\mathcal{L}_\text{grad} + \lambda \mathcal{L}_\text{harm})$. \Comment{Eq~\ref{eq:final_loss}.}
\EndWhile  \\
\Return $\texttt{iFFT}(\mathcal{F_S})$
\end{algorithmic}
\end{algorithm}

\subsection{Baselines} \label{app_baselines}
We compare our method against the following baseline methods:
\begin{itemize}[leftmargin=15pt]
 \item \textbf{Random} is the most straightforward method, which selects a random subsequence.
 \item \textbf{DC} \cite{zhao2021datasetcondensation} matches the model gradient when training with $\mathcal{X}$ and $\mathcal{S}$ at every training step.
 \item \textbf{MTT} \cite{cazenavette2022dataset} matches the training trajectory when training with $\mathcal{X}$ and $\mathcal{S}$, which can be seen as a multistep version of DC.
 \item \textbf{TESLA} \cite{cui2023scaling} simplifies the computation of the trajectory matching loss of MTT and applies soft label assignment.
 \item \textbf{CondTSF} \cite{ding2024condtsf} is the first work to distill the TSF datasets and has achieved state-of-the-art performance by optimizing the value term with the output of expert models in a weighted average manner.
\end{itemize}

\subsection{Evaluation Protocol} \label{app_evalprotocol}
In all experiments, we maintained a consistent input window size ($l$) and output window size ($t$) of 96. We divided all datasets into training, validation, and test sets with a 60:20:20 ratio for ETT datasets and a 70:10:20 ratio for the Electricity, Traffic, and CA datasets. The training set was used for distillation, and we evaluated $\mathcal{S}$ with the validation set every 50 outer epochs for early stopping. Finally, we trained the target models (all backbone models) on the obtained distilled $\mathcal{S}$ and evaluated its performance on the test set. Each experiment was repeated three times, and we report the average Mean Squared Error (MSE), a standard performance metric for TSF.

\begin{table}[t]
    \centering
    \footnotesize
    \begin{tabular}{l c c c c c c}
        \toprule
        & & \multicolumn{5}{c}{\(\lambda\)} \\
        \cmidrule(lr){3-7}
        Dataset & {\(k\)} & {\(10^{-1}\)} & {\(10^{-2}\)} & {\(10^{-3}\)} & {\(10^{-4}\)} & {\(10^{-5}\)} \\
        \midrule
        \multirow{2}{*}{ETTh1} & 48 & 0.481 & 0.434 & 0.446 & 0.441 & 0.470 \\
                               & 96 & 0.471 & 0.432 & 0.446 & \textbf{0.430} & 0.466 \\
        \midrule
        \multirow{2}{*}{ETTm1} & 48 & 0.410 & 0.366 & \textbf{0.364} & 0.372 & 0.405 \\
                               & 96 & 0.425 & 0.378 & 0.369 & 0.380 & 0.397 \\
        \midrule
        \multirow{2}{*}{Electricity} & 48 & 0.256 & 0.209 & 0.212 & 0.215 & 0.235 \\
                                     & 96 & 0.250 & \textbf{0.208} & 0.210 & 0.213 & 0.241 \\
        \midrule
        \multirow{2}{*}{Traffic} & 48 & 0.754 & 0.684 & 0.690 & 0.697 & 0.735 \\
                                 & 96 & 0.745 & \textbf{0.679} & 0.692 & 0.703 & 0.725 \\

        \bottomrule
    \end{tabular}
    \caption{Hyperparameter analysis results (MSE) on various datasets with the DLinear backbone.}
    \label{tab:hyperparam}
\end{table}

\subsection{Implementation Details}\label{implementation}
For the baselines, we used the recommended hyperparameters settings in CondTSF. For HDT, we performed a grid search over the following hyperparameter values: $k=\{M//4, \: M//8\}$ and $\lambda=\{1e\text{-}1, 1e\text{-}2, 1e\text{-}3, 1e\text{-}4, 1e\text{-}5\}$, with the harmonic loss norm set to $p=1$. Our method operates in a channel-independent manner, applying the same distillation process to each channel. The learning rate \(\eta\) for updating \(\mathcal{S}\) was set to 0.01.  The inner loop was run for 20 iterations, with a maximum of 1000 iterations for the outer loop.
HDT was implemented with PyTorch\footnote{https://pytorch.org/}, using the components from the official repositories of DLinear\footnote{https://github.com/vivva/DLinear}, iTransformer\footnote{https://github.com/thuml/iTransformer}, xPatch\footnote{https://github.com/stitsyuk/xPatch}, and CondTSF\footnote{https://github.com/RafaDD/CondTSF}. All training was conducted on a single NVIDIA A100 80GB GPU.

\input{app_table}

\section{Additional Experimental Results}\label{app_add_overall}

\subsection{Hyperparameter Analysis} \label{app_hanalysis}

We conducted a hyperparameter analysis to study the impact of the number of harmonics, $k$, and the loss coefficient, $\lambda$. As shown in Table~\ref{tab:hyperparam}, our analysis indicates that performance is robust to the choice of $\lambda$ within a reasonable range. The optimal number of harmonics, $k$, shows a slight dependency on the dataset. Based on these findings, we selected the hyperparameters that achieved the best performance across the evaluation models.

\subsection{Results for a Smaller Synthetic Dataset}\label{app_additional}
We provide additional experimental results for a smaller synthetic dataset size of $M=192$ in Table~\ref{app_table_performance}. Even at this higher compression ratio, HDT maintains a performance advantage over the baseline methods, further demonstrating the robustness of our approach.

\subsection{Visualization}\label{app_vis}
Figure 4-7 provides a qualitative visualization of the synthetic datasets generated by different methods and backbones on the ETT datasets. For a fair comparison, all methods started from the same initial point (`Real' in the figures). An interesting observation is that the structure of the data distilled by previous methods appears to be influenced by the backbone architecture used. This may suggest a degree of architectural overfitting, where the synthetic data is tailored to the model that created it. Such a dependency could be a drawback in real-world applications where a single distilled dataset is expected to be effective for a diverse range of models.

\clearpage 

\begin{figure*}[t]
  \centering
  \includegraphics[width=\textwidth]{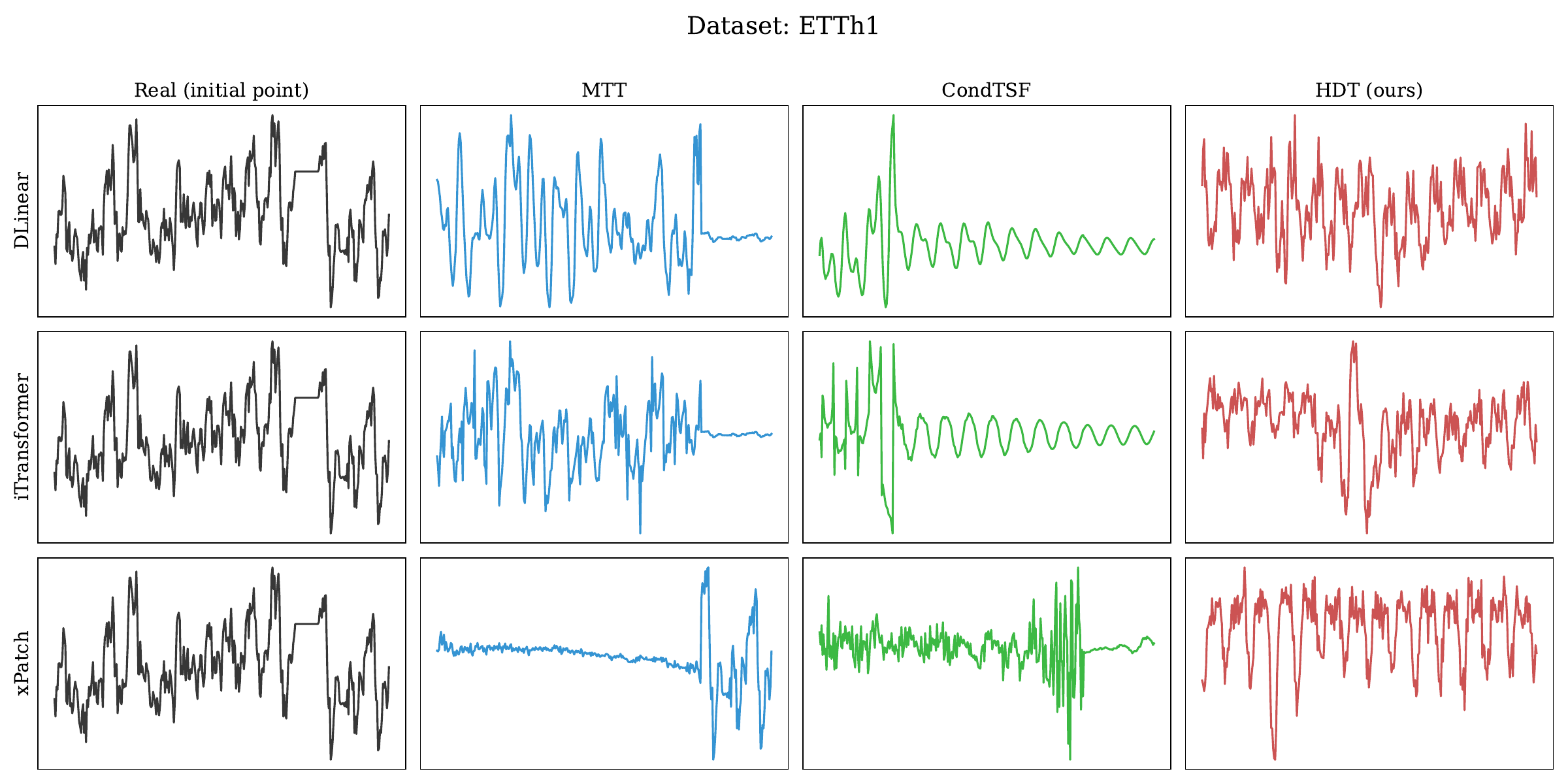}
  \caption{Visual comparison of distilled datasets generated by different methods and backbones on ETTh1. }
\end{figure*}
\begin{figure*}[t]
  \centering
  \includegraphics[width=\textwidth]{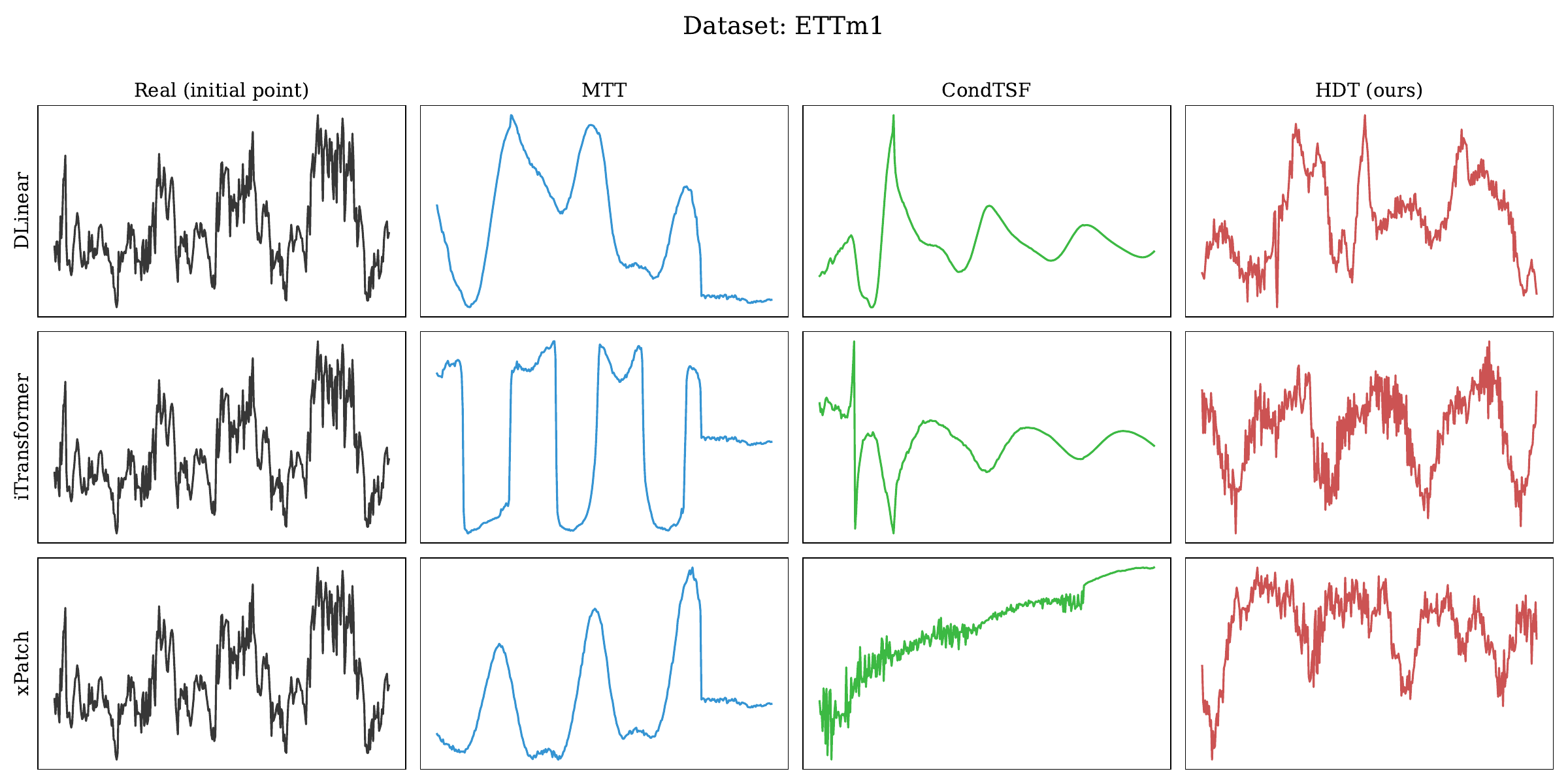}
  \caption{Visual comparison of distilled datasets generated by different methods and backbones on ETTm1. }
\end{figure*}\begin{figure*}[t]
  \centering
  \includegraphics[width=\textwidth]{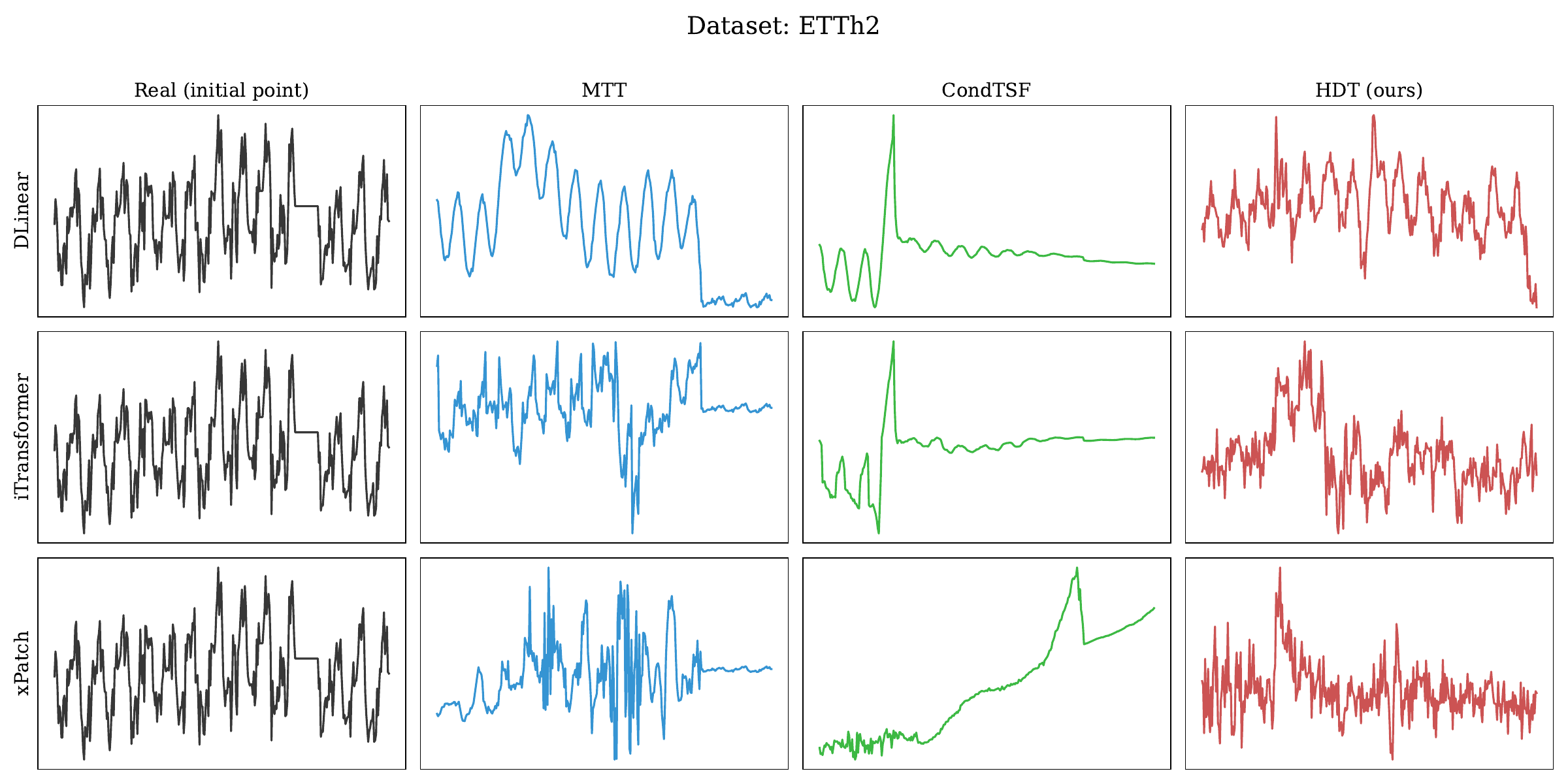}
  \caption{Visual comparison of distilled datasets generated by different methods and backbones on ETTh2. }
\end{figure*}
\begin{figure*}[t]
  \centering
  \includegraphics[width=\textwidth]{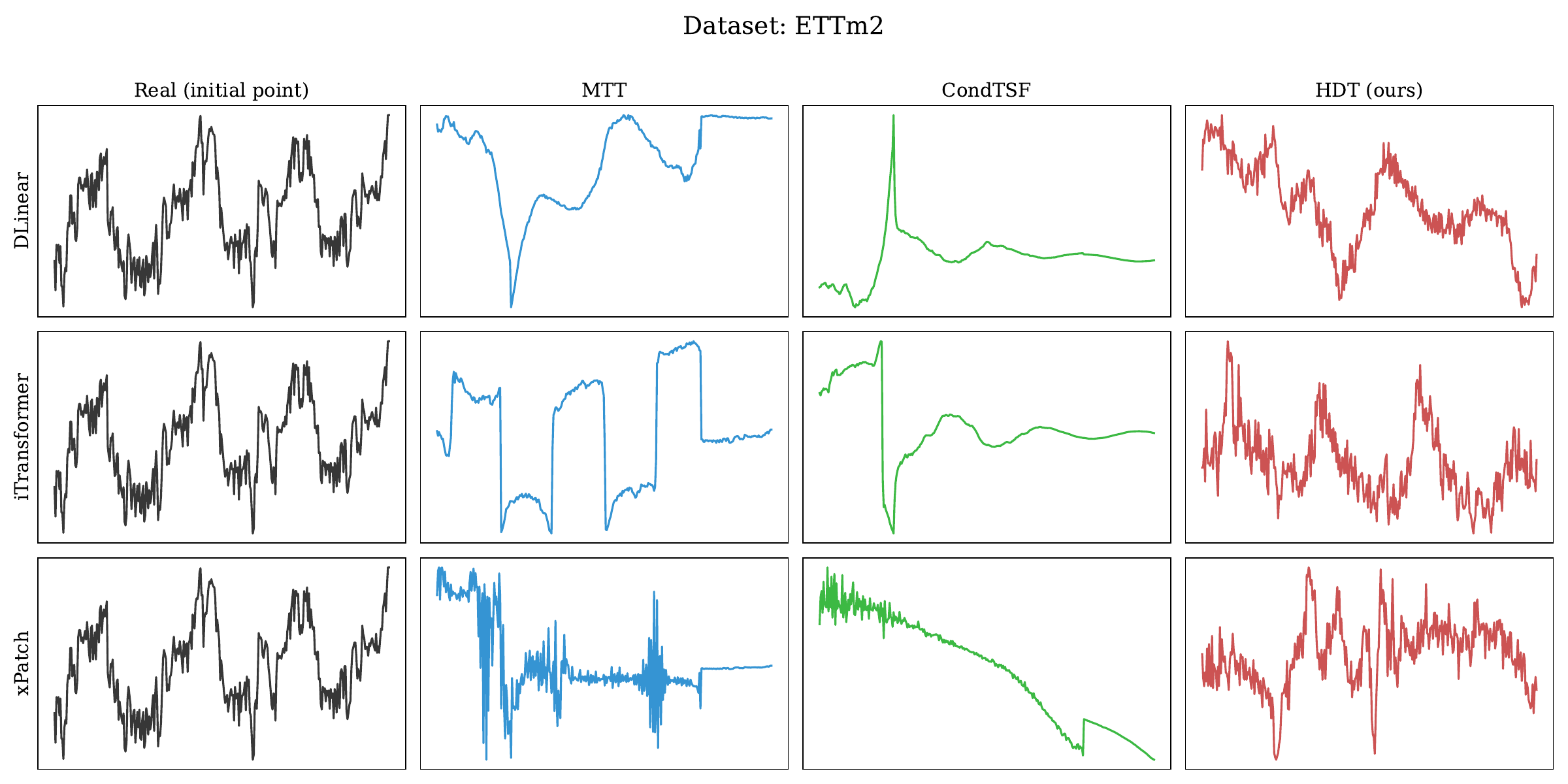}
  \caption{Visual comparison of distilled datasets generated by different methods and backbones on ETTm2. }
\end{figure*}

\end{document}

%% file: total_table.tex
\begin{table*}[t]
\centering

\renewcommand{\arraystretch}{1.0}

\vspace{0.6cm}

\vspace{0.1cm}
\resizebox{\textwidth}{!}{
\begin{tabular}{@{} l ccc ccc ccc ccc ccc ccc@{}}
\toprule
\multicolumn{19}{c}{\textbf{Backbone: DLinear}} \\
\midrule
\multicolumn{1}{r}{\textbf{Dataset}} & 
\multicolumn{3}{c}{ETTh1} & 
\multicolumn{3}{c}{ETTh2} & 
\multicolumn{3}{c}{ETTm1} & 
\multicolumn{3}{c}{ETTm2} & 
\multicolumn{3}{c}{Electricity} & 
\multicolumn{3}{c}{Traffic} \\
\cmidrule(lr){2-4} \cmidrule(lr){5-7} \cmidrule(lr){8-10} \cmidrule(lr){11-13} \cmidrule(lr){14-16} \cmidrule(lr){17-19}
\multicolumn{1}{r}{} & {L} & {T} & {C} 
 & {L} & {T} & {C} 
 & {L} & {T} & {C} 
 & {L} & {T} & {C} 
 & {L} & {T} & {C} 
 & {L} & {T} & {C} \\
\midrule
Random & 0.945 & 0.757 & 0.664 & 1.860 & 0.406 & 0.359 & 0.919 & 0.732 & 0.666 & 1.504 & 0.256 & 0.234 & 0.400 & 0.327 & 0.351 & 1.112 & 0.938 & 0.911 \\
\midrule
	DC&1.193&0.853&0.721&2.200&0.410&0.363&0.801&0.688&0.659&1.670&0.266&0.234&0.413&0.333&0.351&1.122&0.945&0.924 \\ 
	MTT&0.521&0.640&0.587&0.661&0.387&0.346&0.493&0.879&1.342&0.702&0.257&0.248&0.342&0.412&0.489&0.747&1.204&0.852 \\ 
	TESLA&0.576&0.535&0.530&0.724&0.381&0.348&0.443&0.475&0.438&0.428&0.231&0.220&0.388&0.413&0.394&1.087&0.956&0.953 \\ 
	CondTSF&\underline{0.510}&\underline{0.494}&\underline{0.492}&\underline{0.392}&\underline{0.336}&\underline{0.325}&\underline{0.410}&\underline{0.422}&\underline{0.416}&\underline{0.223}&\underline{0.209}&\underline{0.204}&\underline{0.231}&\underline{0.241}&\underline{0.238}&\underline{0.716}&\underline{0.854}&\underline{0.807} \\ 
	HDT (ours)&\textbf{0.430}&\textbf{0.421}&\textbf{0.409}&\textbf{0.359}&\textbf{0.331}&\textbf{0.311}&\textbf{0.364}&\textbf{0.405}&\textbf{0.389}&\textbf{0.211}&\textbf{0.205}&\textbf{0.201}&\textbf{0.208}&\textbf{0.239}&\textbf{0.232}&\textbf{0.679}&\textbf{0.847}&\textbf{0.754} \\ 
\midrule
Full Data & 0.386 & 0.389 & 0.384 & 0.326 & 0.314 & 0.296 & 0.343 & 0.345 & 0.334 & 0.186 & 0.185 & 0.177 & 0.195 & 0.152 & 0.175 & 0.648 & 0.408 & 0.470 \\
\bottomrule
\end{tabular}
}
\resizebox{\textwidth}{!}{
\begin{tabular}{@{} l ccc ccc ccc ccc ccc ccc@{}}
\toprule
\multicolumn{19}{c}{\textbf{Backbone: iTransformer}} \\
\midrule
\multicolumn{1}{r}{\textbf{Dataset}} & 
\multicolumn{3}{c}{ETTh1} & 
\multicolumn{3}{c}{ETTh2} & 
\multicolumn{3}{c}{ETTm1} & 
\multicolumn{3}{c}{ETTm2} & 
\multicolumn{3}{c}{Electricity} & 
\multicolumn{3}{c}{Traffic} \\
\cmidrule(lr){2-4} \cmidrule(lr){5-7} \cmidrule(lr){8-10} \cmidrule(lr){11-13} \cmidrule(lr){14-16} \cmidrule(lr){17-19}
\multicolumn{1}{r}{} & {L} & {T} & {C} 
 & {L} & {T} & {C} 
 & {L} & {T} & {C} 
 & {L} & {T} & {C} 
 & {L} & {T} & {C} 
 & {L} & {T} & {C} \\
\midrule
Random & 0.945 & 0.757 & 0.664 & 1.860 & 0.406 & 0.359 & 0.919 & 0.732 & 0.666 & 1.504 & 0.256 & 0.234 & 0.400 & 0.327 & 0.351 & 1.112 & 0.938 & 0.911 \\
\midrule
	DC&1.228&0.869&0.722&3.200&0.403&0.361&0.871&0.675&0.663&2.472&0.255&0.232&0.462&0.337&0.363&1.000&0.882&0.874 \\ 
	MTT&1.340&\underline{0.452}&\underline{0.441}&7.247&0.375&0.338&1.368&0.584&0.567&4.475&0.251&0.228&0.535&0.316&0.324&\underline{0.881}&0.839&0.943 \\ 
	TESLA&1.079&0.480&0.503&4.914&\underline{0.327}&\underline{0.323}&1.299&0.452&0.428&5.028&0.204&0.205&0.422&0.305&0.351&1.027&0.875&0.861 \\ 
	CondTSF&\underline{0.519}&0.496&0.492&\underline{0.442}&0.334&0.326&\underline{0.411}&\underline{0.411}&\underline{0.407}&\underline{0.236}&\underline{0.200}&\textbf{0.199}&\underline{0.231}&\underline{0.235}&\underline{0.234}&0.965&\underline{0.824}&\underline{0.822} \\ 
	HDT (ours)&\textbf{0.409}&\textbf{0.399}&\textbf{0.397}&\textbf{0.348}&\textbf{0.323}&\textbf{0.320}&\textbf{0.377}&\textbf{0.406}&\textbf{0.399}&\textbf{0.211}&\textbf{0.199}&\underline{0.200}&\textbf{0.229}&\textbf{0.231}&\textbf{0.230}&\textbf{0.755}&\textbf{0.793}&\textbf{0.734} \\ 
\midrule
Full Data & 0.386 & 0.389 & 0.384 & 0.326 & 0.314 & 0.296 & 0.343 & 0.345 & 0.334 & 0.186 & 0.185 & 0.177 & 0.195 & 0.152 & 0.175 & 0.648 & 0.408 & 0.470 \\
\bottomrule
\end{tabular}
}
\resizebox{\textwidth}{!}{
\begin{tabular}{@{} l ccc ccc ccc ccc ccc ccc@{}}
\toprule
\multicolumn{19}{c}{\textbf{Backbone: xPatch}} \\
\midrule
\multicolumn{1}{r}{\textbf{Dataset}} & 
\multicolumn{3}{c}{ETTh1} & 
\multicolumn{3}{c}{ETTh2} & 
\multicolumn{3}{c}{ETTm1} & 
\multicolumn{3}{c}{ETTm2} & 
\multicolumn{3}{c}{Electricity} & 
\multicolumn{3}{c}{Traffic} \\
\cmidrule(lr){2-4} \cmidrule(lr){5-7} \cmidrule(lr){8-10} \cmidrule(lr){11-13} \cmidrule(lr){14-16} \cmidrule(lr){17-19}
\multicolumn{1}{r}{} & {L} & {T} & {C} 
 & {L} & {T} & {C} 
 & {L} & {T} & {C} 
 & {L} & {T} & {C} 
 & {L} & {T} & {C} 
 & {L} & {T} & {C} \\
\midrule
Random & 0.945 & 0.757 & 0.664 & 1.860 & 0.406 & 0.359 & 0.919 & 0.732 & 0.666 & 1.504 & 0.256 & 0.234 & 0.400 & 0.327 & 0.351 & 1.112 & 0.938 & 0.911 \\
\midrule
	DC&1.290&0.829&0.712&3.928&0.408&0.363&0.751&0.657&0.654&2.477&0.250&0.230&0.372&0.322&0.343&1.048&0.882&0.868 \\ 
	MTT&1.740&\underline{0.496}&\underline{0.486}&6.968&\underline{0.359}&\underline{0.329}&0.904&0.589&0.538&5.021&0.240&0.229&0.420&0.334&0.348&1.175&1.025&0.828 \\ 
	TESLA&\underline{1.112}&0.560&0.544&3.850&0.397&0.362&1.161&0.666&\underline{0.421}&3.439&\underline{0.207}&\underline{0.203}&0.508&0.334&0.372&1.003&0.844&0.860 \\ 
	CondTSF&1.172&0.629&0.587&\underline{2.897}&0.406&0.362&\underline{0.574}&\underline{0.550}&0.531&\underline{0.349}&0.248&0.230&\underline{0.237}&\underline{0.261}&\underline{0.246}&\underline{0.904}&\underline{0.827}&\underline{0.808} \\ 
	HDT (ours)&\textbf{0.412}&\textbf{0.416}&\textbf{0.394}&\textbf{0.350}&\textbf{0.349}&\textbf{0.319}&\textbf{0.370}&\textbf{0.523}&\textbf{0.396}&\textbf{0.251}&\textbf{0.204}&\textbf{0.198}&\textbf{0.237}&\textbf{0.253}&\textbf{0.240}&\textbf{0.716}&\textbf{0.787}&\textbf{0.739} \\ 
\midrule
Full Data & 0.386 & 0.389 & 0.384 & 0.326 & 0.314 & 0.296 & 0.343 & 0.345 & 0.334 & 0.186 & 0.185 & 0.177 & 0.195 & 0.152 & 0.175 & 0.648 & 0.408 & 0.470 \\
\bottomrule
\end{tabular}
}
\caption{Overall dataset distillation performance in terms of MSE. The synthetic data size $M$ is set to $384$. \textbf{Bold} indicates the best result, while \underline{underlined} denotes the second-best result.}

\renewcommand{\arraystretch}{1.0}
\label{table:performance}
\end{table*}

%% file: app_table.tex
\begin{table*}[t!]
\centering
\renewcommand{\arraystretch}{1.1}

\vspace{0.1cm}
\resizebox{\textwidth}{!}{
\begin{tabular}{@{} l ccc ccc ccc ccc ccc ccc@{}}
\toprule
\multicolumn{19}{c}{\textbf{Backbone: DLinear}} \\
\midrule
\multicolumn{1}{r}{\textbf{Dataset}} & 
\multicolumn{3}{c}{ETTh1} & 
\multicolumn{3}{c}{ETTh2} & 
\multicolumn{3}{c}{ETTm1} & 
\multicolumn{3}{c}{ETTm2} & 
\multicolumn{3}{c}{Electricity} & 
\multicolumn{3}{c}{Traffic} \\
\cmidrule(lr){2-4} \cmidrule(lr){5-7} \cmidrule(lr){8-10} \cmidrule(lr){11-13} \cmidrule(lr){14-16} \cmidrule(lr){17-19}
\multicolumn{1}{r}{} & {L} & {T} & {C} 
 & {L} & {T} & {C} 
 & {L} & {T} & {C} 
 & {L} & {T} & {C} 
 & {L} & {T} & {C} 
 & {L} & {T} & {C} \\
\midrule
Random & 1.240 & 0.875 & 0.739 & 2.794 & 0.412 & 0.366 & 1.171 & 0.863 & 0.711 & 2.891 & 0.264 & 0.236 & 1.169 & 0.835 & 0.801 & 1.826 & 1.624 & 1.412 \\
\midrule
	DC&1.221&0.896&0.740&3.083&0.414&\underline{0.366}&1.198&0.857&0.720&2.598&0.272&0.238&1.176&0.855&0.801&1.832&1.596&1.382 \\ 
	MTT&1.146&2.489&3.429&3.038&1.971&1.841&1.098&2.582&3.844&2.919&0.333&0.572&0.973&0.933&0.923&1.488&1.479&1.460 \\ 
	TESLA&1.059&1.184&1.393&2.636&0.425&0.382&1.004&0.920&0.828&1.603&0.265&0.237&0.935&0.875&0.867&1.492&1.495&1.476 \\ 
	CondTSF&\underline{0.546}&\underline{0.581}&\underline{0.535}&\underline{0.433}&\underline{0.366}&0.368&\underline{0.436}&\underline{0.484}&\underline{0.476}&\underline{0.244}&\underline{0.222}&\underline{0.222}&\underline{0.228}&\underline{0.287}&\underline{0.246}&\underline{0.718}&\underline{0.863}&\underline{0.810} \\ 
	HDT (ours)&\textbf{0.476}&\textbf{0.481}&\textbf{0.448}&\textbf{0.426}&\textbf{0.335}&\textbf{0.321}&\textbf{0.391}&\textbf{0.410}&\textbf{0.387}&\textbf{0.230}&\textbf{0.210}&\textbf{0.202}&\textbf{0.221}&\textbf{0.271}&\textbf{0.231}&\textbf{0.694}&\textbf{0.845}&\textbf{0.745} \\ 
\midrule
Full Data & 0.386 & 0.389 & 0.384 & 0.326 & 0.314 & 0.296 & 0.343 & 0.345 & 0.334 & 0.186 & 0.185 & 0.177 & 0.195 & 0.152 & 0.175 & 0.648 & 0.408 & 0.470 \\
\bottomrule
\end{tabular}
}
\resizebox{\textwidth}{!}{
\begin{tabular}{@{} l ccc ccc ccc ccc ccc ccc@{}}
\toprule
\multicolumn{19}{c}{\textbf{Backbone: iTransformer}} \\
\midrule
\multicolumn{1}{r}{\textbf{Dataset}} & 
\multicolumn{3}{c}{ETTh1} & 
\multicolumn{3}{c}{ETTh2} & 
\multicolumn{3}{c}{ETTm1} & 
\multicolumn{3}{c}{ETTm2} & 
\multicolumn{3}{c}{Electricity} & 
\multicolumn{3}{c}{Traffic} \\
\cmidrule(lr){2-4} \cmidrule(lr){5-7} \cmidrule(lr){8-10} \cmidrule(lr){11-13} \cmidrule(lr){14-16} \cmidrule(lr){17-19}
\multicolumn{1}{r}{} & {L} & {T} & {C} 
 & {L} & {T} & {C} 
 & {L} & {T} & {C} 
 & {L} & {T} & {C} 
 & {L} & {T} & {C} 
 & {L} & {T} & {C} \\
\midrule
Random & 1.240 & 0.875 & 0.739 & 2.794 & 0.412 & 0.366 & 1.171 & 0.863 & 0.711 & 2.891 & 0.264 & 0.236 & 1.169 & 0.835 & 0.801 & 1.826 & 1.624 & 1.412 \\
\midrule
	DC&1.297&0.865&0.729&3.802&0.409&0.364&1.190&0.807&0.696&\underline{1.570}&0.267&0.237&1.197&0.856&0.778&1.796&1.629&1.383 \\ 
	MTT&1.103&0.779&0.732&3.212&0.361&0.341&1.344&0.730&0.720&3.849&0.229&0.235&1.038&0.803&0.778&1.514&1.432&1.392 \\ 
	TESLA&1.102&0.764&0.710&3.536&0.363&0.346&1.097&0.712&0.693&2.432&0.230&0.223&0.995&0.784&0.768&1.530&1.401&1.343 \\ 
	CondTSF&\underline{0.876}&\underline{0.542}&\underline{0.498}&\underline{3.021}&\underline{0.353}&\underline{0.333}&\underline{0.921}&\underline{0.446}&\underline{0.431}&3.654&\underline{0.211}&\underline{0.204}&\underline{0.364}&\underline{0.260}&\underline{0.277}&\underline{0.782}&\underline{0.835}&\textbf{0.731} \\ 
	HDT (ours)&\textbf{0.524}&\textbf{0.509}&\textbf{0.467}&\textbf{0.402}&\textbf{0.341}&\textbf{0.322}&\textbf{0.565}&\textbf{0.420}&\textbf{0.405}&\textbf{0.261}&\textbf{0.206}&\textbf{0.199}&\textbf{0.242}&\textbf{0.253}&\textbf{0.250}&\textbf{0.767}&\textbf{0.825}&\underline{0.738} \\ 
\midrule
Full Data & 0.386 & 0.389 & 0.384 & 0.326 & 0.314 & 0.296 & 0.343 & 0.345 & 0.334 & 0.186 & 0.185 & 0.177 & 0.195 & 0.152 & 0.175 & 0.648 & 0.408 & 0.470 \\
\bottomrule
\end{tabular}
}
\resizebox{\textwidth}{!}{
\begin{tabular}{@{} l ccc ccc ccc ccc ccc ccc@{}}
\toprule
\multicolumn{19}{c}{\textbf{Backbone: xPatch}} \\
\midrule
\multicolumn{1}{r}{\textbf{Dataset}} & 
\multicolumn{3}{c}{ETTh1} & 
\multicolumn{3}{c}{ETTh2} & 
\multicolumn{3}{c}{ETTm1} & 
\multicolumn{3}{c}{ETTm2} & 
\multicolumn{3}{c}{Electricity} & 
\multicolumn{3}{c}{Traffic} \\
\cmidrule(lr){2-4} \cmidrule(lr){5-7} \cmidrule(lr){8-10} \cmidrule(lr){11-13} \cmidrule(lr){14-16} \cmidrule(lr){17-19}
\multicolumn{1}{r}{} & {L} & {T} & {C} 
 & {L} & {T} & {C} 
 & {L} & {T} & {C} 
 & {L} & {T} & {C} 
 & {L} & {T} & {C} 
 & {L} & {T} & {C} \\
\midrule
Random & 1.240 & 0.875 & 0.739 & 2.794 & 0.412 & 0.366 & 1.171 & 0.863 & 0.711 & 2.891 & 0.264 & 0.236 & 1.169 & 0.835 & 0.801 & 1.826 & 1.624 & 1.412 \\
\midrule
	DC&1.342&0.907&0.729&3.135&0.469&0.352&1.342&0.841&0.707&3.239&0.238&0.227&1.180&0.840&0.806&1.829&1.528&1.396 \\ 
	MTT&1.313&0.843&0.770&3.514&\underline{0.357}&\underline{0.332}&1.195&0.755&0.710&3.580&0.231&0.220&1.082&0.826&0.768&1.694&1.543&1.374 \\ 
	TESLA&1.339&0.813&0.727&\underline{1.725}&0.394&0.362&1.141&0.726&0.694&3.365&0.234&0.216&1.052&0.810&0.764&1.719&1.576&1.378 \\ 
	CondTSF&\underline{0.674}&\underline{0.604}&\underline{0.548}&4.138&0.361&0.336&\underline{0.916}&\underline{0.482}&\underline{0.470}&\underline{1.878}&\underline{0.215}&\underline{0.207}&\underline{0.427}&\underline{0.321}&\underline{0.283}&\underline{1.081}&\underline{1.007}&\underline{0.920} \\ 
	HDT (ours)&\textbf{0.504}&\textbf{0.497}&\textbf{0.451}&\textbf{0.357}&\textbf{0.341}&\textbf{0.319}&\textbf{0.497}&\textbf{0.430}&\textbf{0.411}&\textbf{0.339}&\textbf{0.208}&\textbf{0.196}&\textbf{0.284}&\textbf{0.317}&\textbf{0.279}&\textbf{0.764}&\textbf{0.797}&\textbf{0.717} \\ 
\midrule
Full Data & 0.386 & 0.389 & 0.384 & 0.326 & 0.314 & 0.296 & 0.343 & 0.345 & 0.334 & 0.186 & 0.185 & 0.177 & 0.195 & 0.152 & 0.175 & 0.648 & 0.408 & 0.470 \\
\bottomrule
\end{tabular}
}

\caption{Overall distillation performance on the synthetic data size $M=192$ in terms of MSE.  Bold indicates the best result, while underlined denotes the second-best result.}

\renewcommand{\arraystretch}{1.0}
\label{app_table_performance}
\end{table*}

%% file: aaai2026.bib
@inproceedings{timemixer,
title={TimeMixer: Decomposable Multiscale Mixing for Time Series Forecasting},
author={Shiyu Wang and Haixu Wu and Xiaoming Shi and Tengge Hu and Huakun Luo and Lintao Ma and James Y. Zhang and JUN ZHOU},
booktitle={The Twelfth International Conference on Learning Representations},
year={2024},
url={https://openreview.net/forum?id=7oLshfEIC2}
}

@inproceedings{Fred,
  title={Frequency Domain-based Dataset Distillation}, 
    author={Donghyeok Shin and Seungjae Shin and Il-Chul Moon},
  booktitle={Proceedings of the Advances in Neural Information Processing Systems},
  year={2023}
}

@inproceedings{
xu2024fits,
title={{FITS}: Modeling Time Series with 10k Parameters},
author={Zhijian Xu and Ailing Zeng and Qiang Xu},
booktitle={The Twelfth International Conference on Learning Representations},
year={2024},
url={https://openreview.net/forum?id=bWcnvZ3qMb}
}

@misc{cisco,
  author       = {Cisco},
  title        = {Cisco for Oil and Gas},
  howpublished = {\url{https://www.cisco.com/site/us/en/solutions/industries/energy/oil-gas/index.html}},
  note         = {Accessed: July 21, 2025},
year = 2025
}

@misc{noaa,
  author       = {NOAA},
  title        = {Open Data Dissemination (NODD)},
  howpublished = {https://www.noaa.gov/information-technology/open-data-dissemination},
  note         = {Accessed: July 21, 2025},
year = 2025
}

@article{vitaldb,
  author       = {Lee, Hyeongki and Jung, Kyeongman and Kang, Soojeong and Yoo, Sungsik and Park, Rok and Park, Woojin and Lee, Hyung-Chul and Jheon, Sanghoon and Kim, Younsuck and Lee, Hyeonjeong and others},
  title        = {VitalDB, a high-fidelity multi-parameter vital signs database in surgical patients},
  journal      = {Scientific Data},
  volume       = {9},
  number       = {1},
  pages        = {1--11},
  year         = {2022},
  publisher    = {Nature Publishing Group},
  doi          = {10.1038/s41597-022-01411-5}
}

@misc{survey,
      title={A Survey of Deep Learning and Foundation Models for Time Series Forecasting}, 
      author={John A. Miller and Mohammed Aldosari and Farah Saeed and Nasid Habib Barna and Subas Rana and I. Budak Arpinar and Ninghao Liu},
      year={2024},
      eprint={2401.13912},
      archivePrefix={arXiv},
      primaryClass={cs.LG},
      url={https://arxiv.org/abs/2401.13912}, 
}

@article{autoreg,
  title={Evolutionary spectra and non-stationary processes},
  author={Priestley, Maurice B},
  journal={Journal of the Royal Statistical Society: Series B (Methodological)},
  volume={27},
  number={2},
  pages={204--229},
  year={1965},
  publisher={Wiley Online Library}
}

@article{kullback,
  title={On the Kullback-Leibler information divergence of locally stationary processes},
  author={Dahlhaus, Rainer},
  journal={Stochastic processes and their applications},
  volume={62},
  number={1},
  pages={139--168},
  year={1996},
  publisher={Elsevier}
}

@inproceedings{largest,
 author = {Liu, Xu and Xia, Yutong and Liang, Yuxuan and Hu, Junfeng and Wang, Yiwei and BAI, LEI and Huang, Chao and Liu, Zhenguang and Hooi, Bryan and Zimmermann, Roger},
 booktitle = {Proceedings of the Advances in Neural Information Processing Systems},
 editor = {A. Oh and T. Naumann and A. Globerson and K. Saenko and M. Hardt and S. Levine},
 pages = {75354--75371},
 publisher = {Curran Associates, Inc.},
 title = {LargeST: A Benchmark Dataset for Large-Scale Traffic Forecasting},
 url = {https://proceedings.neurips.cc/paper_files/paper/2023/file/ee57cd73a76bd927ffca3dda1dc3b9d4-Paper-Datasets_and_Benchmarks.pdf},
 volume = {36},
 year = {2023}
}

@article{wiener,
  title={Generalized harmonic analysis},
  author={Wiener, Norbert},
  journal={Acta mathematica},
  volume={55},
  number={1},
  pages={117--258},
  year={1930},
  publisher={Springer}
}

@article{khintchine,
  title={Korrelationstheorie der station{\"a}ren stochastischen Prozesse},
  author={Khintchine, Alexander},
  journal={Mathematische Annalen},
  volume={109},
  number={1},
  pages={604--615},
  year={1934},
  publisher={Springer}
}

@misc{timesfm,
      title={A Decoder-only Foundation Model for Time-series Forecasting}, 
      author={Abhimanyu Das and Weihao Kong and Rajat Sen and Yichen Zhou},
      year={2024},
  booktitle={Proceedings of the International Conference on Machine Learning},
  year={2024}
}

@misc{moirai,
      title={Unified Training of Universal Time Series Forecasting Transformers}, 
      author={Gerald Woo and Chenghao Liu and Akshat Kumar and Caiming Xiong and Silvio Savarese and Doyen Sahoo},
      year={2024},
  booktitle={Proceedings of the International Conference on Machine Learning},
  year={2024}
}

@inproceedings{jin2022graph,
  title={Graph Condensation for Graph Neural Networks},
  author={Jin, Wei and Zhao, Lingxiao and Zhang, Shichang and Liu, Yozen and Tang, Jiliang and Shah, Neil},
  booktitle={Proceedings of the International Conference on Learning Representations},
  year={2022}
}

@inproceedings{liu2024gcsr,
  title={Graph Data Condensation via Self-Expressive Graph Structure Reconstruction},
  author={Liu, Zhanyu and Zeng, Chaolv and Zheng, Guanjie},
  booktitle={Proceedings of the ACM SIGKDD Conference on Knowledge Discovery and Data Mining},
  year={2024}
}

@inproceedings{maekawa2023text,
  title={Dataset Distillation with Attention Labels for Fine-Tuning BERT},
  author={Maekawa, Aru and Kobayashi, Naoki and Funakoshi, Kotaro and Okumura, Manabu},
  booktitle={Proceedings of the Annual Meeting of the Association for Computational Linguistics},
  year={2023}
}

@inproceedings{maekawa2024dilm,
  title={{DiLM}: Distilling Dataset into Language Model for Text-Level Dataset Distillation},
  author={Maekawa, Aru and Kosugi, Satoshi and Funakoshi, Kotaro and Okumura, Manabu},
  booktitle={Proceedings of the Annual Conference of the North American Chapter of the Association for Computational Linguistics},
  year={2024}
}

@article{ding2024condtsf,
  title={CondTSF: One-Line Plugin of Dataset Condensation for Time Series Forecasting},
  author={Ding, Jianrong and Liu, Zhanyu and Zheng, Guanjie and Jin, Haiming and Kong, Linghe},
  journal={Proceedings of the Advances in Neural Information Processing Systems},
  year={2024}
}

@article{liu2024dataset,
  title={Dataset Condensation for Time Series Classification via Dual Domain Matching},
  author={Liu, Zhanyu and Hao, Ke and Zheng, Guanjie and Yu, Yanwei},
  journal={Proceedings of the ACM SIGKDD Conference on Knowledge Discovery and Data Mining},
  year={2024}
}

@article{wang2018datasetdistillation,
  title={Dataset Distillation},
  author={Wang, Tongzhou and Zhu, Jun-Yan and Torralba, Antonio and Efros, Alexei A.},
  journal={arXiv Preprint arXiv:1811.10959},
  year={2018}
}

@inproceedings{zhao2021datasetcondensation,
  title={Dataset Condensation with Gradient Matching},
  author={Zhao, Bo and Bilen, Hakan},
  booktitle={Proceedings of the International Conference on Learning Representations},
  year={2021}
}

@inproceedings{nguyen2021kip,
  title={Dataset Meta-Learning from Kernel Ridge-Regression},
  author={Nguyen, Timothy and Chen, Zhourong and Lee, Jaehoon},
  booktitle={Proceedings of the International Conference on Learning Representations},
  year={2021}
}

@article{nguyen2021dataset,
  title={Dataset Distillation with Infinitely Wide Convolutional Networks},
  author={Nguyen, Timothy and Novak, Roman and Xiao, Lechao and Lee, Jaehoon},
  journal={Proceedings of the Advances in Neural Information Processing Systems},
  year={2021}
}

@inproceedings{zhou2022dataset,
  title={Dataset Distillation Using Neural Feature Regression},
  author={Zhou, Yongchao and Nezhadarya, Ehsan and Ba, Jimmy},
  booktitle={Proceedings of the Advances in Neural Information Processing Systems},
  year={2022}
}

@inproceedings{zhao2023distribution,
  title={Dataset Condensation with Distribution Matching},
  author={Zhao, Bo and Bilen, Hakan},
  booktitle={Proceedings of the IEEE/CVF Winter Conference on Applications of Computer Vision},
  year={2023}
}

@inproceedings{cazenavette2022dataset,
  title={Dataset Distillation by Matching Training Trajectories},
  author={Cazenavette, George and Wang, Tongzhou and Torralba, Antonio and Efros, Alexei A. and Zhu, Jun-Yan},
  booktitle={Proceedings of the IEEE/CVF Conference on Computer Vision and Pattern Recognition},
  year={2022}
}

@inproceedings{cui2023scaling,
  title={Scaling Up Dataset Distillation to ImageNet-1K with Constant Memory},
  author={Cui, Justin and Wang, Ruochen and Si, Si and Hsieh, Cho-Jui},
  booktitle={Proceedings of the International Conference on Machine Learning},
  year={2023}
}

@inproceedings{guo2024datm,
  title={Towards Lossless Dataset Distillation via Difficulty-Aligned Trajectory Matching},
  author={Guo, Ziyao and Wang, Kai and Cazenavette, George and Li, Hui and Zhang, Kaipeng and You, Yang},
  booktitle={Proceedings of the International Conference on Learning Representations},
  year={2024}
}

@inproceedings{zhang2024m3d,
  title={{M3D}: Dataset Condensation by Minimizing Maximum Mean Discrepancy},
  author={Zhang, Hansong and Li, Shikun and Wang, Pengju and Zeng, Dan Ge and Shiming},
  booktitle={Proceedings of the AAAI Conference on Artificial Intelligence},
  year={2024}
}

@inproceedings{cazenavette2023generalizing,
  title={Generalizing Dataset Distillation via Deep Generative Prior},
  author={Cazenavette, George and Wang, Tongzhou and Torralba, Antonio and Efros, Alexei A. and Zhu, Jun-Yan},
  booktitle={Proceedings of the IEEE/CVF Conference on Computer Vision and Pattern Recognition},
  year={2023}
}

@inproceedings{liu2023mgdd,
  title={{MGDD}: A Meta Generator for Fast Dataset Distillation},
  author={Liu, Songhua and Wang, Xinchao},
  booktitle={Proceedings of the Advances in Neural Information Processing Systems},
  year={2023}
}

@inproceedings{liu2023fewshot,
  title={Few-Shot Dataset Distillation via Translative Pre-Training},
  author={Liu, Songhua and Wang, Xinchao},
  booktitle={Proceedings of the IEEE/CVF International Conference on Computer Vision},
  year={2023}
}

@inproceedings{wei2023sparse,
  title={Sparse Parameterization for Epitomic Dataset Distillation},
  author={Wei, Xing and Cao, Anjia and Yang, Funing and Ma, Zhiheng},
  booktitle={Proceedings of the Advances in Neural Information Processing Systems},
  year={2023}
}

@misc{xpatch,
      title={xPatch: Dual-Stream Time Series Forecasting with Exponential Seasonal-Trend Decomposition}, 
      author={Artyom Stitsyuk and Jaesik Choi},
      year={2025},
      eprint={2412.17323},
      archivePrefix={arXiv},
      primaryClass={cs.LG},
      url={https://arxiv.org/abs/2412.17323}, 
}

@inproceedings{
itransformer,
title={iTransformer: Inverted Transformers Are Effective for Time Series Forecasting},
author={Yong Liu and Tengge Hu and Haoran Zhang and Haixu Wu and Shiyu Wang and Lintao Ma and Mingsheng Long},
booktitle={Proceedings of the International Conference on Learning Representations},
year={2024},
url={https://openreview.net/forum?id=JePfAI8fah}
}

@inproceedings{donghao2024moderntcn,
  title={{ModernTCN}: A Modern Pure Convolution Structure for General Time Series Analysis},
  author={Luo, Donghao and Xue, Wang},
  booktitle={Proceedings of the International Conference on Learning Representations},
  year={2024},
  url={https://openreview.net/forum?id=vpJMJerXHU}
}

@inproceedings{lstnet,
  author = {Lai, Guokun and Chang, Wei-Cheng and Yang, Yiming and Liu, Hanxiao},
  title = {Modeling Long- and Short-Term Temporal Patterns with Deep Neural Networks},
  year = {2018},
  isbn = {9781450356572},
  publisher = {Association for Computing Machinery},
  address = {New York, NY, USA},
  url = {https://doi.org/10.1145/3209978.3210006},
  doi = {10.1145/3209978.3210006},
  booktitle = {Proceedings of the International ACM SIGIR Conference on Research \& Development in Information Retrieval},
  location = {Ann Arbor, MI, USA},
  series = {SIGIR '18}
}

@inproceedings{zeng2023transformers,
  title={Are Transformers Effective for Time Series Forecasting?},
  author={Zeng, Ailing and Chen, Muxi and Zhang, Lei and Xu, Qiang},
  booktitle={Proceedings of the AAAI Conference on Artificial Intelligence},
  year={2023}
}

@inproceedings{ekambaram2023tsmixer,
  title={TSMixer: Lightweight MLP-Mixer Model for Multivariate Time Series Forecasting},
  author={Ekambaram, Vijay and Jati, Arindam and Nguyen, Nam and Sinthong, Phanwadee and Kalagnanam, Jayant},
  booktitle={Proceedings of the ACM SIGKDD Conference on Knowledge Discovery and Data Mining},
  year={2023}
}

@article{informer,
  title={Informer: Beyond Efficient Transformer for Long Sequence Time-Series Forecasting},
  url={https://ojs.aaai.org/index.php/AAAI/article/view/17325},
  DOI={10.1609/aaai.v35i12.17325},
  journal={Proceedings of the AAAI Conference on Artificial Intelligence},
  author={Zhou, Haoyi and Zhang, Shanghang and Peng, Jieqi and Zhang, Shuai and Li, Jianxin and Xiong, Hui and Zhang, Wancai},
  year={2021},
  month={May},
  pages={11106--11115}
}

@article{nie2023time,
  title={A Time Series Is Worth 64 Words: Long-Term Forecasting with Transformers},
  author={Nie, Yuqi and Nguyen, Nam H and Sinthong, Phanwadee and Kalagnanam, Jayant},
  journal={Proceedings of the International Conference on Learning Representations},
  year={2023}
}

@article{wu2022timesnet,
  title={TimesNet: Temporal 2D-Variation Modeling for General Time Series Analysis},
  author={Wu, Haixu and Hu, Tengge and Liu, Yong and Zhou, Hang and Wang, Jianmin and Long, Mingsheng},
  journal={arXiv Preprint arXiv:2210.02186},
  year={2022},
  publisher={arXivPreprint}
}
